\documentclass{article}

\usepackage[authoryear,round]{natbib}
\usepackage{arxiv}
\usepackage[utf8]{inputenc} 
\usepackage[T1]{fontenc}    
\usepackage{hyperref}       
\usepackage{url}            
\usepackage{booktabs}       
\usepackage{amsfonts}       
\usepackage{nicefrac}       
\usepackage{microtype}      
\usepackage{amsmath,amssymb,amsfonts,mathtools}
\usepackage{amsthm}
\usepackage{graphicx}
\usepackage{xcolor}
\usepackage{enumerate}
\usepackage{enumitem}
\usepackage{bbm}
\usepackage{xparse}
\usepackage{steinmetz}
\usepackage{varwidth}
\usepackage{cleveref}
\usepackage{multirow}
\usepackage{refcount}
\usepackage{url}
\usepackage{upgreek}
\usepackage{multirow}
\usepackage{scrextend}
\usepackage{tikz-network}
\usepackage{array}
\usepackage{tabularx}
\usepackage{subcaption}
\usepackage{longtable}
\hypersetup{
    colorlinks=true,
    allcolors=black,
    pdfborder={0 0 0}
}
\usepackage{makecell}
\usepackage{algorithm}
\usepackage{algpseudocode}
\usepackage{appendix}

\newcommand{\bla}{\color{black}}
\usetikzlibrary{arrows,decorations.pathmorphing,backgrounds,positioning,fit,petri,snakes}

\allowdisplaybreaks

\newtheorem{definition}{Definition}[section]
\newtheorem{theorem}{Theorem}[section]
\newtheorem{proposition}{Proposition}[section]
\newtheorem{exmp}{Example}[section]
\newtheorem{remark}{Remark}[section]
\newtheorem{lemma}{Lemma}[section]

\newtheorem{corollary}{Corollary}[section]
\title{The Edge-Based Contiguous \(p\)-Median Problem: Exact Formulations and Strengthening Inequalities}

\author{
 Zeyad Kassem \\
  School of Computing and Augmented Intelligence\\
  Arizona State University\\
  Tempe, AZ 85281, USA \\
  \texttt{zekassem@asu.edu} \\
  \And
 Behnam Jabbari-Marand\thanks{Corresponding author.} \\
  Edward P. Fitts Department of Industrial and Systems Engineering\\
  North Carolina State University\\
  Raleigh, NC 27606, USA \\
  \texttt{bjabbar@ncsu.edu} \\
  \And
 Adolfo R. Escobedo \\
  Edward P. Fitts Department of Industrial and Systems Engineering\\
  North Carolina State University\\
  Raleigh, NC 27606, USA \\
  \texttt{arescobedo@ncsu.edu}
}
\makeatletter
\def\@date{}
\makeatother
\begin{document}
\maketitle
\begin{abstract}
This paper introduces the edge-based contiguous $p$-median (EC$p$M) problem, which partitions the roads in a network into a given number of compact and contiguous territories. We develop two binary programming formulations that minimize the total network distance from each road to its assigned center node. The first uses an exponential number of cut-set contiguity (CSC) constraints and is solved with a branch-and-cut (B\&C) algorithm that checks incumbent solutions for contiguity and adds violated CSC constraints as needed. The second is a compact formulation that imposes a stronger form of contiguity through shortest-path contiguity (SPC) constraints. These constraints are added up front, and the formulation is solved directly with an off-the-shelf solver. Despite this stronger requirement, we show that SPC-based EC$p$M has the same optimal objective value as EC$p$M with general contiguity and even the relaxed edge-based $p$-median (E$p$M). As a result, the SPC constraints preserve at least one optimal E$p$M solution and can therefore strengthen the E$p$M formulation. We also establish three structural properties of EC$p$M solutions and use them to derive three families of inequalities that further strengthen these formulations. Computational experiments on $84$ instances from $14$ real-world road networks show that adding the SPC constraints to E$p$M enables all instances to be solved to optimality and reduces solution times. In contrast, E$p$M without these constraints and CSC-based EC$p$M reach the time limit on six and eight instances, respectively. The proposed inequalities reduce both counts to one and further reduce solution times.
\end{abstract}


\section{Introduction}
\label{s:intro}
The rapid global growth of e-commerce in recent years has increased interest in last-mile freight logistics, which refers to the final stage of a consumer supply chain in which consignments are transported from a distribution hub directly to consumers or collection points \citep{bosona2020urban}. The last mile is the least efficient part of the supply chain and can account for up to 28\% of total transportation costs \citep{bergmann2020integrating}. Moreover, last-mile logistics can contribute up to 30\% of vehicle miles and up to 50\% of greenhouse gas (GHG) emissions in developed urban areas \citep{dablanc2010impacts}. Prior studies show that the proximity of warehouses to service areas is a major determinant of operational and environmental costs \citep{wygonik2018urban}. In addition, the consolidation of shipments to customers along road segments and routes can help achieve economies of density. This, in turn, enables the use of larger, more cost-effective vehicles, which spreads fixed costs over more units and lowers variable costs, thereby reducing the unit shipment cost \citep{alumur2021perspectives}.

Discrete location-allocation models provide a systematic framework for incorporating such strategic considerations into the design of service areas for last-mile logistics. These models determine where to locate a set of facilities or services (e.g., warehouses, schools, hospitals) and how to allocate customers or communities to the selected locations according to one or more decision criteria \citep{Marianov2002}. Graphs are commonly used to model the underlying territories, including road, air transportation, and maritime networks \citep{Tansel1983}. Location-allocation models have also been adapted to accommodate different distance functions \citep{Brimberg2008}. Those that use network distances are especially well suited to capture practical aspects of road travel in last-mile logistics \citep{kalcsics2020districting}. A fundamental model in this class is the $p$-median problem, first introduced by \citet{Hakimi1964}. It selects $p$ facilities from a discrete set of candidate locations, typically represented by nodes of the underlying graph, and assigns customers to these facilities to minimize the total weighted distance between customers and their assigned facilities \citep{Hakimi1964}.

The $p$-median problem is NP-hard \citep{hakimi1979algorithmic}, which has motivated substantial research on model development and the design of effective solution methods. Comprehensive reviews of the facility location literature are provided by \citet{owen1998strategic}, \citet{melo2009facility}, and \citet{ulukan2015survey}, while textbook treatments include \citet{drezner2004facility} and \citet{Daskin2013}. As in \citet{Hakimi1964}, many studies use network distances (or functions of them) between facilities and customers as proxies for operational costs. The $p$-median problem also has several practical extensions. For example, the capacitated $p$-median problem limits the amount of demand that each facility can serve \citep{ceselli2005branch}. The dynamic $p$-median problem allows facilities to be opened, closed, relocated, or mobilized over a planning horizon and jointly optimizes location and timing decisions to minimize the total cost of serving allocated demand \citep{guden2019dynamic}.

Various extensions of $p$-median models capture a wide range of practical considerations \citep{croci2023balanced}, but some important aspects remain relatively unexplored. The $p$-median literature predominantly considers node-based models, in which customers and facilities are located at the vertices of the underlying graph, such as street intersections. Edge-based models can better represent residential customers located along streets, yet they have received considerably less attention \citep{kalcsics2020districting}. Although research in computer science considers edge-based graph partitioning---particularly for applications such as distributed preprocessing of large graphs---these models are not designed around the requirements of logistics applications. Moreover, existing solution approaches for edge-based graph partitioning primarily rely on heuristic algorithms \citep{zhang2017graph,mayer2018adwise,mayer2021hybrid}. Another practical aspect of $p$-median models that has received much less attention is contiguity, which requires each customer to be connected to its assigned facility through the territory served by that facility. A notable exception is the node-based $p$-median extension of \citet{lolonis1993location}, which enforces contiguity through a set of nonlinear expressions. Because the resulting formulation is difficult to solve directly, the authors develop a heuristic solution approach. Other extensions focus only on connectivity among selected facilities or impose stronger routing requirements, rather than explicitly enforcing contiguity within assigned territories.

To enforce contiguity explicitly and efficiently, we introduce linear contiguity constraints for the \textit{edge-based $p$-median} problem that require the selected center of each territory and its assigned edges to form a connected subgraph. Although this form of contiguity is commonly imposed in spatial optimization models \citep{williams2002zero,carvajal2013imposing,SalazarAguilar2011,GarciaAyala2016}, it has received less attention in edge-based location-allocation models without additional districting requirements. This distinction is important because additional design criteria used in more restrictive spatial models may conflict with the basic structure of $p$-median models. For example, requiring each facility to serve approximately the same amount of demand, a common condition in logistics districting, can compromise the compactness of the service areas \citep{GarciaAyala2016}. We therefore develop binary programming formulations and exact solution approaches for the edge-based $p$-median problem and its contiguous extension and analyze their theoretical and computational properties.

This paper makes several contributions to territorial design. First, it introduces the edge-based contiguous $p$-median (EC$p$M) problem and develops two exact binary programming formulations. Both formulations partition a road network into a fixed number of compact and contiguous territories by locating centers and assigning roads to them to minimize the total distance between roads and their assigned centers. The first formulation enforces contiguity through an exponential family of cut-set contiguity (CSC) constraints, while the second is a compact formulation with shortest-path contiguity (SPC) constraints. Second, we develop a branch-and-cut (B\&C) algorithm for the CSC formulation that checks incumbent solutions for contiguity and generates violated constraints as needed; we solve the SPC formulation directly with an off-the-shelf branch-and-bound solver. Third, we show that, despite its stronger requirement, SPC-based EC$p$M has the same optimal objective value as EC$p$M with general contiguity and the relaxed edge-based $p$-median (E$p$M) problem. Fourth, we establish three structural properties of EC$p$M solutions and use them to derive three families of model-strengthening inequalities, along with a practical strategy for selecting promising inequalities to add to the formulations. Finally, we conduct computational experiments on $84$ instances derived from $14$ real-world road networks with up to $3{,}462$ nodes and $3{,}809$ edges to evaluate the formulations, the SPC constraints, and the proposed inequalities.

The remainder of the paper is organized as follows. Section~\ref{background} reviews the relevant background and related work. Section~\ref{sec:1} introduces the E$p$M and EC$p$M problems, develops their binary programming formulations, and establishes theoretical properties of the SPC constraints. Section~\ref{sec:Logic_cuts} derives three families of model-strengthening inequalities and presents a practical strategy for selecting those that are most promising. Section~\ref{sec:computational_results} reports computational experiments that evaluate the formulations, the SPC constraints, and the proposed inequalities. Finally, Section~\ref{Conclusion} concludes the paper with a summary of the main results and directions for future work.

\section{Background}
\label{background}
Location-allocation models create territories based on criteria such as compactness and contiguity. A territory is generally considered \textit{compact} if it is nearly circular or square-shaped, undistorted, without holes, and has smooth boundaries \citep{Butsch2014}. Compact territories tend to reduce travel distances \citep{GarciaAyala2016}; therefore, incorporating compactness into location-allocation models can improve subsequent routing operations for e-commerce companies. For example, \citet{alvarez2025multi} introduce a lexicographic compactness metric and special boundary conditions for designing postal delivery territories. Contiguity is another important design criterion for last-mile logistics. A territory is considered \textit{contiguous} if any two points within it can be connected without leaving the territory. Contiguity can help e-commerce companies group nearby customer orders and apply routing algorithms separately within each non-overlapping territory \citep{alvarez2021districting}. For example, \citet{lespay2022territory} design contiguous and compact territories for multi-period last-mile delivery while explicitly accounting for vehicle routing and customer time windows. Contiguity is also one of the most difficult criteria to enforce in spatial optimization models \citep{ricca2008local}, and compactness alone does not guarantee contiguity \citep{Shirabe2009}.

Extensions of the $p$-median problem vary in their connectivity requirements. The connected $p$-median problem requires the selected medians (i.e., facilities) to form a connected subgraph, rather than requiring each assigned territory to be contiguous \citep{yen2010connected}. Other extensions enforce territorial contiguity indirectly through more restrictive routing requirements. One example is the Hamiltonian $p$-median (H$p$M) problem, which combines the $p$-median problem with the traveling salesperson problem (TSP) \citep{branco1990hamiltonian,barbato2024hamiltonian}. H$p$M selects $p$ depots, assigns customers to them, and determines service routes to minimize total travel distance. It is related to, but more difficult than, the generalized vehicle routing problem, in which depot locations are fixed \citep{gollowitzer2011new}. More recently, \citet{canas2025location} develop strengthened formulations and valid inequalities for H$p$M, together with exact branch-and-cut algorithms. In H$p$M, subtour-elimination constraints require each vehicle to complete a tour from its depot that visits every assigned customer exactly once, thereby imposing a stronger requirement than territory contiguity alone. Logistics districting also builds on the $p$-median problem but imposes contiguity and workload-balance requirements, which further increase computational difficulty. These models should be distinguished from political districting models, which span several decades of research (e.g., \citet{hess1965nonpartisan,mehrotra1998optimization,ricca2011political,validi2022imposing}). Although both settings consider compactness, contiguity, and balance, they differ substantially in the units partitioned, distance measures, and application-specific criteria.

Methods for enforcing contiguity have been developed for applications such as political districting \citep{garfinkel1970optimal,validi2022imposing}, land acquisition \citep{williams2002zero}, forest planning \citep{carvajal2013imposing}, and territorial districting \citep{SalazarAguilar2011,GarciaAyala2016}. One common strategy is to adapt connectivity constraints from related optimization problems, such as the subtour-elimination constraints of the TSP \citep{dantzig1954solution}. These constraints prevent disconnected cycles formed over proper subsets of the cities. They can be expressed through \textit{cut-set} logic that enforces connectivity between every proper subset of nodes and its complement. Variants of this approach have been developed for both node-based \citep{SalazarAguilar2011} and edge-based \citep{GarciaAyala2016} districting models. Because the number of cut-set constraints grows exponentially with the problem size, including all of them explicitly is often impractical. As a result, these models are typically paired with separation schemes.

Three main separation schemes can be used to enforce contiguity among facilities and customers: iterative branch-and-bound with cut generation (B\&B\&Cut), integer separation, and fractional separation. B\&B\&Cut is an iterative exact approach that first solves the model without contiguity constraints using branch-and-bound (B\&B) \citep{SalazarAguilar2011}. A graph-search algorithm, such as breadth-first or depth-first search, then checks the resulting solution for disconnected territories. If contiguity violations are found, the corresponding cuts are added and B\&B is restarted; otherwise, the solution is feasible for the original model and the procedure terminates. Integer separation avoids these repeated restarts by checking incumbent integer solutions for contiguity during the B\&B search and adding violated constraints as needed, for example, through lazy-constraint callbacks \citep{mendes2022capacitated}. This approach requires greater familiarity with solver-specific features, particularly callback mechanisms. Similar to integer separation, fractional separation is also a branch-and-cut approach \citep{martin2001general}, but it generates cuts at fractional solutions during B\&B. This may require more complex or computationally expensive separation procedures, such as repeated minimum-cut computations \citep{validi2022imposing}. Consequently, prior studies have relied primarily on B\&B\&Cut or integer separation \citep{SalazarAguilar2011,GarciaAyala2016,mendes2022capacitated,validi2022imposing}.

Previous research has also developed polynomial-size constraint sets to impose contiguity. These approaches fall into two main categories: network flow-based \citep{Shirabe2009,jafari2013new,Wang2021} and neighbor-based \citep{mehrotra1998optimization,onal2016optimal,farughi2020novel}. In flow-based formulations, each undirected edge of the underlying graph is replaced by two oppositely directed arcs, and contiguity is enforced by requiring flow among nodes assigned to the same territory. These models generally scale poorly and are effective only for relatively small instances \citep{Wang2021,validi2022imposing}. Moreover, multiple flow configurations can represent the same allocation within a territory, which introduces substantial symmetry. Flow-based models are therefore often strengthened with symmetry-reducing cuts \citep{Wang2021,mendes2022capacitated}. Neighbor-based formulations instead rely on local assignment logic. If a node is assigned to a given center, at least one neighboring node must also be assigned to that center. Repeated application of this requirement creates a chain of adjacent nodes that connects each assigned node to its center. These constraints can be viewed as a special application of dominance constraints developed for the simple plant location problem \citep{canovas2007strengthened}. To the best of our knowledge, these contiguity formulations have not been directly compared with cut-set-based formulations, nor have their structural properties been systematically analyzed to derive model-strengthening inequalities.

\section{Exact Formulations}
\label{sec:1}

In this section, we define the EC$p$M problem and present a binary programming model with cut-set-based contiguity constraints. These constraints are exponential in number, and including all of them explicitly requires prohibitive computational resources. Therefore, we complement this model with a branch-and-cut (B\&C) algorithm that typically generates only a small subset of the constraints on real-world instances. We also propose an alternative model that incorporates shortest-path contiguity (SPC) constraints, which are polynomial in number. In addition, we show that the SPC constraints are supervalid inequalities for E$p$M (i.e., EC$p$M without the contiguity requirement).

\subsection{The Edge-based $p$-median (E$p$M) Model}
Before introducing EC$p$M, we first define the E$p$M problem, which determines the locations of $p$ center nodes and the allocation of edges to them. The underlying road network is represented by a connected undirected planar graph $G=(V,E)$. Let $\ell_{(m,n)}$ denote the length of edge $(m,n)\in E$. We next introduce the network distance concepts used throughout this work.
\begin{definition}\label{Shortest_Path_Definition} 
Let $SP_{i,j}$ denote the set of edges on a shortest path between nodes $i$ and $j$. This set is expressed mathematically as
\begin{align}
SP_{i,j} = \mathop{\mathrm{arg\,min}}_{\rho \in \mathcal{P}_{i,j}}
\left\{ \sum_{(m,n) \in \rho} \ell_{(m,n)} \right\}, \nonumber
\end{align}
where $\mathcal{P}_{i,j}$ is the set of all possible paths between nodes $i,j \in V$.
\end{definition}
\begin{definition}\label{def_2}
Let $SP_{i,(j,k)}$ denote the set of edges on a shortest path from node $i$ to edge $(j,k)$. It is given by either $SP_{i,j}$ or $SP_{i,k}$, depending on whether the path reaches $(j,k)$ at $j$ or $k$. This set is expressed mathematically as
\begin{align}
SP_{i,(j,k)} =
\mathop{\mathrm{arg\,min}}_{\rho \in \{\,SP_{i,j},\,SP_{i,k}\,\}}
\;\sum_{(m,n)\in \rho} \ell_{(m,n)} .\nonumber
\end{align}
\end{definition}
\begin{definition}\label{def_3}
Let ${d_{i,\left(j,k\right)}}$ denote the network distance given by the sum of the positive edge lengths along a shortest path from node $i$ to edge $(j,k)$. This distance is expressed mathematically as 
\begin{align}
{d_{i,\left(j,k\right)}} = \sum_{(m,n) \in SP_{i,(j,k)}} \ell_{(m,n)}.\nonumber
\end{align}
\end{definition}
The shortest path between every pair of nodes in $G$ can be computed efficiently using Dijkstra's algorithm \citep{Ahuja1988}. Based on Definitions \ref{Shortest_Path_Definition} and \ref{def_2}, this result extends to the shortest path between every node and edge, which must be computed to construct the E$p$M model. The model components are described next.\vspace{3pt}

\noindent \emph{Sets and Indices}:\vspace{-6pt}
\begin{labeling}{$g\in G$\hspace{20pt}}
	\item[$i \in V$] Nodes in the network (corresponding to street crossings, dead ends, etc.).
	\item[$(j,k) \in E$] Edges in the network (corresponding to roads, streets, etc.).
\end{labeling} 
\noindent \emph{Parameters}:\vspace{-6pt}
\begin{labeling}{$\omega^{\max}_{ij}$\hspace{25pt}}
	\item[${d_{i,\left(j,k\right)}}$]   Distance from node $i$ to edge $(j,k)$ (see Definition \ref{def_3}).
	\item[$p$]    Number of territories.
\end{labeling}
\noindent \emph{Decision Variables}:\vspace{-12pt}\\
\begin{labeling}{$P_{g}^{\max}, P_{g}^{\min}$\hspace{1pt}}
	\item[${w}_i$] $= \begin{cases}
	1 & \text{if node $i$ is selected as a center node}, \\[-4pt]
	0 & \text{otherwise}.
	\end{cases}$
	\item[$x_{i,\left(j,k\right)}$]  $=\begin{cases}
	1 & \text{if edge $\left(j,k\right)$ is allocated to center node $i$} , \\[-4pt]
	0 & \text{otherwise}.
	\end{cases}$\vspace{-6pt}
\end{labeling}
\begin{flalign}
    &\hspace{0pt}\displaystyle
    \text{(E$p$M)}
    && \makebox[3em][l]{min}\displaystyle \sum_{i \in V} \sum_{\left(j,k\right) \in E} {d_{i,\left(j,k\right)}} {x_{i,\left(j,k\right)}}
    && \label{obj}\\
    &&& \makebox[3em][l]{s.t.}\displaystyle \sum_{i \in V} {x_{i,\left(j,k\right)}}=1
    && \forall \left(j,k\right) \in E \label{const:110}\\
    &&& \makebox[3em][l]{}\displaystyle \sum_{i \in V} {w_i}=p
    && \label{const:111}\\
    &&& \makebox[3em][l]{}\displaystyle {x_{i,\left(j,k\right)}} \leq w_i
    && \forall i \in V,\ \forall (j,k) \in E \label{const:1121}\\
    &&& \makebox[3em][l]{}\displaystyle {x_{i,\left(j,k\right)}} \in \{0,1\}
    && \forall i \in V,\ \forall \left(j,k\right) \in E \label{const:10}\\
    &&& \makebox[3em][l]{}\displaystyle {w_i} \in \{0,1\}
    && \forall i \in V \label{const:11}
\end{flalign}

Objective function~\eqref{obj} maximizes compactness by minimizing the dispersion between each center node $i$ and its allocated edges. Constraints~\eqref{const:110} ensure that each edge $(j,k)\in E$ is allocated to exactly one center node. Constraint \eqref{const:111} ensures that exactly $p$ center nodes are selected. Constraints~\eqref{const:1121} allow edge $(j,k)$ to be allocated to node $i$ only if $i$ is selected as a center node. Constraints~\eqref{const:10} and~\eqref{const:11} specify the domains of the variables.

Next, Example~\ref{expmot} and Figure~\ref{Fig:mot} illustrate that model \eqref{obj}--\eqref{const:11} can yield solutions with non-contiguous territories. The following subsection presents two binary programming models that extend E$p$M to exclude such solutions; the models differ in how they impose contiguity.

\begin{exmp}\label{expmot}
Consider the connected planar graph $G=(V,E)$ depicted in Figure~\ref{Fig:mot1}, where $|V|=9$, $|E|=11$, and $p=2$. Figures~\ref{Fig:mot2} and~\ref{Fig:mot3} display two E$p$M solutions. In each case, solid edges correspond to the first territory, whereas dotted edges correspond to the second territory; the centers of the two territories follow the same formatting scheme as their corresponding edges. In both solutions, nodes $3$ and $1$ are the centers of the first and second territories, respectively. The sole difference is the allocation of edge $(2,7)$, which is allocated to the first territory in Figure~\ref{Fig:mot2} and to the second territory in Figure~\ref{Fig:mot3}. The two solutions have the same dispersion ($35$), but only the latter is contiguous.
\end{exmp}

\subsection{The Edge-based Contiguous $p$-median Model with Cut-Set Contiguity Constraints (EC$p$M-CSC)}\label{EPMC_Model}
As an extension of the E$p$M problem, this subsection presents a binary programming formulation that partitions $E$ into $p$ compact and contiguous territories. The formulation imposes contiguity by requiring every cut set that separates any pair of nonadjacent edges allocated to a given territory to contain at least one edge allocated to the same territory. Cut-set contiguity constraints, which are exponential in number, are analogous to the subtour-elimination constraints of the TSP (see Section \ref{background}). They were first developed for node-based districting by \cite{drexl1999fast}, who also explain that real-world instances can often be solved by imposing only a small subset of these constraints. The constraints were adapted for edge-based districting in \cite{GarciaAyala2016}; however, that model assumes that the centers are predetermined (i.e., it is an allocation-only model).

\noindent \emph{Sets and Indices}:
\begin{labeling}{$g\in G$\hspace{20pt}}
	\item[$S \subset E$]     A subset of edges. 
	\item[$V(S) \subset V$]  Set of nodes that are incident to any edge in $S$.
	\item[$\sigma(S)$]  Cut set of $S$  (i.e., $\sigma(S):= \left\lbrace\left(i,j\right) \in E| i \in V\left(S\right), j \in V \setminus V\left(S\right) \right\rbrace$).
\end{labeling} 
\begin{flalign}
&\hspace{0pt}\displaystyle
\text{(EC$p$M-CSC)}
&& \min \sum_{i \in V} \sum_{\left(j,k\right) \in E} {d_{i,\left(j,k\right)}} {x_{i,\left(j,k\right)}}
&& \tag{\ref{obj}}\\
&&& \text{s.t.}\quad \eqref{const:110}-\eqref{const:11}
&& \nonumber\\
&&& \sum_{s \in \sigma(S)}\! {x_{is}}\! - \sum_{s \in S}\! {x_{is}}\! \geq {x_{i,(j,k)}}\! - |S|
&& \forall i \in V,\forall (j,k)\! \in E,\ S\! \subset E\! \setminus\! \big(\sigma(\left\lbrace (j,k)\!\right\rbrace)\cup \{(j,k)\}\big)
& \label{ECPM_2}
\end{flalign}
\begin{figure}
    \centering
    \begin{subfigure}[b]{0.35\textwidth}
        \scalebox{0.75}{
            \begin{tikzpicture}
                \Vertex[x=0,y=0,label=1,color=white]{A}
                \Vertex[x=0,y=2,label=2,color=white]{B}
                \Vertex[x=-2,y=0,label=3,color=white]{C}
                \Vertex[x=2,y=0,label=4,color=white]{D}
                \Vertex[x=0,y=-2,label=5,color=white]{E}
                \Vertex[x=-2,y=-2,label=6,color=white]{F}
                \Vertex[x=-2,y=2,label=7,color=white]{G}
                \Vertex[x=2,y=2,label=8,color=white]{H}
                \Vertex[x=-2,y=3.5,label=9,color=white]{I}
                \Edge[lw=2pt,label=5,position=left](A)(B)
                \Edge[lw=2pt,label=5,position=above](A)(C)
                \Edge[lw=2pt,label=5,position=above](A)(D)
                \Edge[lw=2pt,label=5,position=right](A)(E)
                \Edge[lw=2pt,label=5,position=left](C)(F)
                \Edge[lw=2pt,label=5,position=above](E)(F)
                \Edge[lw=2pt,label=5,position=left](C)(G)
                \Edge[lw=2pt,label=5,position=above](B)(G)
                \Edge[lw=2pt,label=5,position=above](B)(H)
                \Edge[lw=2pt,label=5,position=right](H)(D)
                \Edge[lw=2pt,label=2,position=left](G)(I)
            \end{tikzpicture}
        }
        \caption{Road network to be partitioned into two territories with E$p$M.} \label{Fig:mot1}
    \end{subfigure}\hfill
    \begin{subfigure}[b]{0.3\textwidth}
        \scalebox{0.75}{
            \begin{tikzpicture}
                \Vertex[x=0,y=0,label=1,color=white,style=dotted]{A}
                \Vertex[x=0,y=2,label=2,color=white]{B}
                \Vertex[x=-2,y=0,label=\textcolor{white}{3},color=black]{C}
                \Vertex[x=2,y=0,label=4,color=white]{D}
                \Vertex[x=0,y=-2,label=5,color=white]{E}
                \Vertex[x=-2,y=-2,label=6,color=white]{F}
                \Vertex[x=-2,y=2,label=7,color=white]{G}
                \Vertex[x=2,y=2,label=8,color=white]{H}
                \Vertex[x=-2,y=3.5,label=9,color=white]{I}
                \Edge[lw=2pt,style=dotted](A)(B)
                \Edge[lw=2pt](A)(C)
                \Edge[lw=2pt,style=dotted](A)(D)
                \Edge[lw=2pt,style=dotted](A)(E)
                \Edge[lw=2pt](C)(F)
                \Edge[lw=2pt](E)(F)
                \Edge[lw=2pt,style=dotted](C)(G)
                \Edge[lw=2pt](B)(G)
                \Edge[lw=2pt,style=dotted](B)(H)
                \Edge[lw=2pt,style=dotted](H)(D)
                \Edge[lw=2pt,style=dotted](G)(I)
            \end{tikzpicture}
        }
        \caption{Non-contiguous solution\\ (dispersion $=35$).} \label{Fig:mot2}
    \end{subfigure}\hfill
    \begin{subfigure}[b]{0.3\textwidth}
        \scalebox{0.75}{
            \begin{tikzpicture}
                \Vertex[x=0,y=0,label=1,color=white,style=dotted]{A}
                \Vertex[x=0,y=2,label=2,color=white]{B}
                \Vertex[x=-2,y=0,label=\textcolor{white}{3},color=black]{C}
                \Vertex[x=2,y=0,label=4,color=white]{D}
                \Vertex[x=0,y=-2,label=5,color=white]{E}
                \Vertex[x=-2,y=-2,label=6,color=white]{F}
                \Vertex[x=-2,y=2,label=7,color=white]{G}
                \Vertex[x=2,y=2,label=8,color=white]{H}
                \Vertex[x=-2,y=3.5,label=9,color=white]{I}
                \Edge[lw=2pt,style=dotted](A)(B)
                \Edge[lw=2pt](A)(C)
                \Edge[lw=2pt,style=dotted](A)(D)
                \Edge[lw=2pt,style=dotted](A)(E)
                \Edge[lw=2pt](C)(F)
                \Edge[lw=2pt](E)(F)
                \Edge[lw=2pt,style=dotted](C)(G)
                \Edge[lw=2pt,style=dotted](B)(G)
                \Edge[lw=2pt,style=dotted](B)(H)
                \Edge[lw=2pt,style=dotted](H)(D)
                \Edge[lw=2pt,style=dotted](G)(I)
            \end{tikzpicture}
        }
        \caption{Contiguous solution\\ (dispersion $=35$).} \label{Fig:mot3}
    \end{subfigure}
    \caption{E$p$M solutions with identical dispersion (i.e., objective function value); note: solid lines denote edges allocated to the first territory, and dotted lines denote edges allocated to the second territory.}
    \label{Fig:mot}
\end{figure}
The main difference from E$p$M is the addition of constraints~\eqref{ECPM_2}, which ensure that edges allocated to the same center node form a connected territory. To elaborate, let $S \subset E \setminus \big(\sigma(\left\lbrace (j,k)\right\rbrace)\cup\{(j,k)\}\big)$ be any subset of edges that contains neither edge $(j,k)$ nor any edge adjacent to it, where $\sigma(\left\lbrace (j,k)\right\rbrace)$ denotes the set of edges adjacent to $(j,k)$. If edge $(j,k)$ is not allocated to center node $i$ (i.e., $x_{i,(j,k)}=0$), the corresponding constraint is always satisfied. Similarly, if at least one edge in $S$ is not allocated to center node $i$, the second summation is strictly less than $|S|$, and the constraint is again always satisfied. Thus, the constraint is activated only when edge $(j,k)$ and all edges in $S$ are allocated to center node $i$. In this case, it requires at least one edge in the cut set $\sigma(S)$ to be allocated to center node $i$. Applying these constraints for every possible subset $S$ ensures that all edges within each territory are connected.

\subsubsection{Branch-and-Cut Algorithm}
The number of constraints~\eqref{ECPM_2} grows exponentially with the problem size, making their full implementation impractical. Next, we pair EC$p$M-CSC with an integer separation scheme, namely a branch-and-cut (B\&C) algorithm. The algorithm, which is adapted from \cite{mendes2022capacitated}, generates only those cut-set contiguity constraints that are deemed necessary based on the solution to a relaxed version of EC$p$M. These constraints are identified using a separation algorithm. In the presented B\&C implementation, the cuts corresponding to \eqref{ECPM_2} are added as \textit{lazy constraints} as needed, thereby avoiding the need to restart B\&B to exclude discontiguous incumbent solutions.

The procedure begins by solving a relaxed version of the original model in which all contiguity constraints are omitted. This relaxation may yield incumbent solutions that violate contiguity. During the B\&B search, each incumbent binary solution $\mathbf{x^{*}}$ is analyzed to detect contiguity violations. Let $G[E_k]=(V_k,E_k)$ denote the subgraph of $G$ induced by the set of edges $E_k$ assigned to territory $k$, where $V_k := \left\lbrace u,v \in V \mid (u,v) \in E_k \right\rbrace$. A graph traversal algorithm, such as breadth-first search (BFS), is used to determine the connected components of each $G[E_k]$. Contiguity is violated when multiple connected components exist within a territory.

For each of the $p$ territories identified in the solution, the method checks whether the corresponding subgraph $G[E_k]$ is connected. If not, a set $R$ of violated contiguity constraints is generated. For every pair of connected components $S_k'$ and $S_k''$ in $E_k$, the procedure considers each edge $(l,m) \in S_k'$ and adds a constraint requiring at least one edge in the cut set of $S_k''$ to be assigned to the same territory as $(l,m)$. The same procedure is repeated for each edge $(u,v) \in S_k''$ with respect to $S_k'$. The B\&B search continues until an optimal solution satisfying all contiguity constraints is found.

\subsection{The Edge-based Contiguous $p$-median Model with Shortest-path Contiguity Constraints (EC$p$M-SPC)}
\label{derive}
An alternative approach to modeling contiguity in spatial optimization applications is to leverage shortest paths. \cite{zoltners1983sales} introduce the concept of hierarchical adjacency trees for the sales territory alignment problem with predetermined centers. For each predetermined center node, a set of shortest paths is determined from the center to all other nodes, with the predecessor-successor relationships along each shortest path recorded. Contiguity is enforced by ensuring that no node can be allocated to a center node unless one or more of its immediate predecessors along a shortest path is also allocated to the same center node. \cite{cova2000contiguity} derive shortest-path contiguity-$n$ constraints for single-region site search problems. These constraints allow each spatial unit to be reached from the root along only one of its $n$ shortest paths; to that end, it is necessary to compute the shortest, second-shortest, and so on, up to the $n{\text{th}}$-shortest path from each spatial unit to the root, where $n$ is assumed to be reasonably small. \cite{mehrotra1998optimization} develop and implement a special version of shortest-path contiguity-$n$ constraints for political districting in which only the first shortest path is considered. More recently, \cite{farughi2020novel} introduce a distinct set of shortest-path contiguity constraints for designing health care districts. The constraints ensure that if nodes $j$ and $k$ are allocated to center node $i$, then all nodes along $SP_{j,k}$ must also be assigned to center node $i$. A common feature of these shortest-path constraints is that they are defined exclusively for node-based territorial planning models. An exception is \cite{cova2000contiguity}, where the spatial unit is a raster, which can nevertheless be readily represented as a node.

To the best of our knowledge, shortest-path contiguity constraints---henceforth referred to as SPC constraints---have not been defined for edge-based location-allocation models. This work seeks to fill this research gap. The main intuition behind the proposed constraints is that, for any connected planar graph with more than one edge, if edge $(j,k)$ is allocated to center node $i$, then all edges along the shortest path that connects them must also be allocated to center node $i$. Since $SP_{i,(j,k)}$ is an uninterrupted path, requiring all of its edges to be allocated to center node $i$ whenever edge $(j,k)$ is allocated to $i$ ensures contiguity within each territory.

This requirement can be modeled using three logically equivalent constraint formulations, denoted SPC-1, SPC-2, and SPC-3; their expressions are given by
\small
\begin{flalign*}
\text{(SPC-1)                          }
     &|SP_{i,(j,k)}| x_{i,\left(j,k\right)}\leq \sum_{\left(l,m\right) \in SP_{i, (j,k)}} x_{i,\left(l,m\right)} \text{      } &&\forall i \in V, \forall \left(j,k\right) \in E  \tag{$8^{\prime}$}\label{eqn:ineq_new}\\[-7pt]
	 \text{(SPC-2)                   } 
	& x_{i,\left(j,k\right)}\leq  x_{i,\left(l,m\right)} &&\forall i \in V, \forall \left(j,k\right) \in E, \forall (l,m) \in SP_{i, (j,k)} \tag{$8^{\prime\prime}$} \label{eqn:ineq_new_1}\\
	 \text{(SPC-3)                }
	& x_{i,\left(j,k\right)}\leq  x_{i,\left(l,m\right)}  &&\forall i \in V, \forall \left(j,k\right) \in E, \forall (l,m) \in SP_{i, (j,k)}\cap \sigma(\left\lbrace (j,k)\right\rbrace) \tag{$8^{\prime\prime\prime}$},\label{eqn:ineq_new_2}
 \end{flalign*}\normalsize
where $\sigma(\left\lbrace (j,k)\right\rbrace)$ is the cut set of edge $(j,k)$ (i.e., the set of edges adjacent to it). Note that in \eqref{eqn:ineq_new_2}, $(l,m)$ is the edge immediately preceding $(j,k)$ on the shortest path from node $i$ to edge $(j,k)$, because it belongs to both $SP_{i,(j,k)}$ and $\sigma(\left\lbrace (j,k)\right\rbrace)$.

Next, we establish that constraint sets SPC-1, SPC-2, and SPC-3 are logically equivalent and show that SPC-2 and SPC-3 yield formulations of equal tightness, both at least as tight as that obtained with SPC-1. The proofs of Theorems \ref{thm:tight} and \ref{thm:tighter} are provided in Appendices \ref{append:2} and \ref{append:3}, respectively.

\begin{theorem}
For any instance of EC$p$M, constraint sets SPC-1, SPC-2, and SPC-3 are logically equivalent.
\label{logically}
\end{theorem}

\begin{proof} \label{append:1}
If $x_{i,\left(j,k\right)}=0$, all three left-hand sides become 0, making the constraints redundant. If $x_{i,\left(j,k\right)}=1$, each constraint set imposes the shortest-path allocation requirement. SPC-1 forces the full set of edges along $SP_{i,(j,k)}$ to be allocated to node $i$ using a single inequality. SPC-2 relates, through separate expressions, the allocation of each $(l,m)\in SP_{i,(j,k)}$ to center $i$ with the corresponding allocation variable $x_{i,\left(j,k\right)}$. SPC-3 follows an approach similar to SPC-2, but it requires only that the edge preceding $\left(j,k\right)$ on $SP_{i,(j,k)}$ be allocated to $i$.
\end{proof}

\begin{theorem} Define the polytope $\mathcal{P}_\text{SPC-1}=\lbrace \mathbf{(x,w)} \in \left[ 0,1 \right]^{|V|\times |E|} \times \left[ 0,1 \right]^{|V|} :\eqref{const:110}-\eqref{const:1121}, \eqref{eqn:ineq_new}\rbrace$ and the polytope $\mathcal{P}_\text{SPC-2}=\lbrace \mathbf{(x,w)} \in \left[ 0,1 \right]^{|V|\times |E|} \times \left[ 0,1 \right]^{|V|} :\eqref{const:110}-\eqref{const:1121}, \eqref{eqn:ineq_new_1}\rbrace$. For any instance of EC$p$M, $\mathcal{P}_\text{SPC-2} \subseteq \mathcal{P}_\text{SPC-1}$, and the inclusion can be strict.
\label{thm:tight}\end{theorem}
\begin{theorem} Let $\mathcal{P}_\text{SPC-2}$ be defined as in Theorem~\ref{thm:tight}, and define the polytope $\mathcal{P}_\text{SPC-3}=\lbrace \mathbf{(x,w)} \in \left[ 0,1 \right]^{|V|\times |E|} \times \left[ 0,1 \right]^{|V|} \text{: }\eqref{const:110}-\eqref{const:1121}, \eqref{eqn:ineq_new_2}\rbrace$. Then, $\mathcal{P}_\text{SPC-2} = \mathcal{P}_\text{SPC-3}$.
\label{thm:tighter}\end{theorem}
There are $\mathcal{O}\left(|V|^2|E|\right)$ SPC-2 constraints and $\mathcal{O}\left(|V||E|\right)$ SPC-1 and SPC-3 constraints; hence, all three yield compact formulations. Because SPC-3 induces a linear programming relaxation that is at least as tight as that obtained with SPC-1 (see Theorems \ref{thm:tight} and \ref{thm:tighter}), SPC-3 is adopted for the remainder of this work. Therefore, the EC$p$M-SPC formulation is given by:
\begin{flalign}
&\hspace{0pt}\displaystyle
\text{(EC$p$M-SPC)}
&& \min \sum_{i \in V} \sum_{\left(j,k\right) \in E} {d_{i,\left(j,k\right)}} {x_{i,\left(j,k\right)}}
&& \tag{\ref{obj}}\\
&&& \text{s.t.}\quad \eqref{const:110}-\eqref{const:11}
&& \nonumber\\
&&& x_{i,\left(j,k\right)}\leq x_{i,\left(l,m\right)}
&& \forall i \in V,\ \forall \left(j,k\right) \in E,\ \forall (l,m) \in SP_{i,(j,k)}\cap \sigma(\left\lbrace (j,k)\right\rbrace)
& \tag{$8^{\prime\prime\prime}$} \label{ECPM_2-SPC}
\end{flalign}

Contiguity is not required by the E$p$M problem. Thus, at first glance, the SPC constraints may seem extraneous to this simpler model. Nevertheless, we show that they can also expedite its solution. To that end, it is useful to distinguish among three types of cutting-plane techniques that can be imposed to reduce the solution space. The best-known and most widely applied type, \textit{valid inequalities} (alternatively referred to as valid logic cuts), can remove only fractional points from the relaxed feasible region, meaning that they do not remove any integer-feasible points. On the other hand, \emph{non-valid logic cuts} can make deeper cuts to the solution space by additionally allowing the removal of integer-feasible, non-optimal solutions. This type of cutting plane was first introduced by \cite{Hooker1994a} and has been applied to capacitated warehouse location \citep{Osorio1999}, the design of chemical processing networks \citep{Hooker1994a}, and multilevel generalized assignment \citep{Osorio2003}. The third type of cutting-plane technique, known as \emph{supervalid inequalities}, also allows the removal of non-optimal solutions that are integer-feasible \citep{israeli2002shortest}. However, they go one step further by allowing the removal of some, but not all, optimal solutions. Thus, supervalid inequalities can make larger cuts to the solution space than the first two approaches. This type of cutting plane has been applied to interdiction problems \citep{sefair2017defender,baycik2019robust,wei2024supervalid} and routing problems \citep{yuan2020note,yuan2021mixed}.

Theorem~\ref{thm:3} shows that the SPC constraints are supervalid inequalities for E$p$M. Before presenting the theorem, the following example motivates the key insight behind this additional use of the SPC constraints (see Figure~\ref{fig:main}).
\begin{exmp}\label{exp1}
Consider the connected planar graph $G=(V,E)$ depicted in Figure~\ref{Fig10}, where $|V|=9$, $|E|=11$, and $p=2$. Figure~\ref{Fig5} depicts a non-contiguous E$p$M solution, and Figure~\ref{Fig6} shows an E$p$M solution satisfying the SPC constraints. In each case, solid edges correspond to the first territory, whereas dotted edges correspond to the second territory; the centers of the two territories follow the same formatting scheme as the corresponding edges. In both solutions, nodes $1$ and $6$ are the centers. The sole difference is the allocation of edge $(2,7)$: in Figure~\ref{Fig5}, it is allocated to the first territory, whereas in Figure~\ref{Fig6}, it is allocated to the second territory. The latter allocation dominates the former because $d_{1,(2,7)}=5 < 10 = d_{6,(2,7)}$. The difference stems from the fact that allocating edge $(7,9)$ to node $1$ and enforcing the SPC constraints forces all edges along $SP_{1,(7,9)}$, including $(2,7)$, to be allocated to center node $1$. The following theorem generalizes the insight illustrated by this example.
\end{exmp}

\begin{theorem}\label{thm:3}
Let $G=(V,E)$ be an undirected connected graph with $|E|\geq2$ and $\ell_e>0$ for all $e\in E$. Then, E$p$M admits an optimal solution $(w^*,x^*)$ for which all SPC constraints hold. Specifically, if edge $(j,k)$ is assigned to center $i$, i.e., $x^*_{i,(j,k)}=1$, then all edges in $SP_{i,(j,k)}$ are also assigned to $i$.
\end{theorem}

\begin{footnotesize}
\begin{figure}
    \centering
    \begin{subfigure}[b]{0.35\textwidth}
        \scalebox{0.75}{
            \begin{tikzpicture}
                \Vertex[x=0,y=0,label=1,color=white]{A}
                \Vertex[x=0,y=2,label=2,color=white]{B}
                \Vertex[x=-2,y=0,label=3,color=white]{C}
                \Vertex[x=2,y=0,label=4,color=white]{D}
                \Vertex[x=0,y=-2,label=5,color=white]{E}
                \Vertex[x=-2,y=-2,label=6,color=white]{F}
                \Vertex[x=-2,y=2,label=7,color=white]{G}
                \Vertex[x=2,y=2,label=8,color=white]{H}
                \Vertex[x=-2,y=3.5,label=9,color=white]{I}
                \Edge[lw=2pt,label=5,position=left](A)(B)
                \Edge[lw=2pt,label=5,position=above](A)(C)
                \Edge[lw=2pt,label=5,position=above](A)(D)
                \Edge[lw=2pt,label=5,position=right](A)(E)
                \Edge[lw=2pt,label=5,position=left](C)(F)
                \Edge[lw=2pt,label=5,position=above](E)(F)
                \Edge[lw=2pt,label=5,position=left](C)(G)
                \Edge[lw=2pt,label=5,position=above](B)(G)
                \Edge[lw=2pt,label=5,position=above](B)(H)
                \Edge[lw=2pt,label=5,position=right](H)(D)
                \Edge[lw=2pt,label=2,position=left](G)(I)
            \end{tikzpicture}
        }
        \caption{Road network to be partitioned into two territories with E$p$M.} \label{Fig10}
    \end{subfigure}\hfill
    \begin{subfigure}[b]{0.3\textwidth}
        \scalebox{0.75}{
            \begin{tikzpicture}
                \Vertex[x=0,y=0,label=1,color=white,style=dotted]{A}
                \Vertex[x=0,y=2,label=2,color=white]{B}
                \Vertex[x=-2,y=0,label=3,color=white]{C}
                \Vertex[x=2,y=0,label=4,color=white]{D}
                \Vertex[x=0,y=-2,label=5,color=white]{E}
                \Vertex[x=-2,y=-2,label=\textcolor{white}{6},color=black]{F}
                \Vertex[x=-2,y=2,label=7,color=white]{G}
                \Vertex[x=2,y=2,label=8,color=white]{H}
                \Vertex[x=-2,y=3.5,label=9,color=white]{I}
                \Edge[lw=2pt,style=dotted](A)(B)
                \Edge[lw=2pt,style=dotted](A)(C)
                \Edge[lw=2pt,style=dotted](A)(D)
                \Edge[lw=2pt,style=dotted](A)(E)
                \Edge[lw=2pt](C)(F)
                \Edge[lw=2pt](E)(F)
                \Edge[lw=2pt](C)(G)
                \Edge[lw=2pt](B)(G)
                \Edge[lw=2pt,style=dotted](B)(H)
                \Edge[lw=2pt,style=dotted](H)(D)
                \Edge[lw=2pt,style=dotted](G)(I)
            \end{tikzpicture}
        }
        \caption{E$p$M solution before imposing SPC\\ (dispersion = $35$).} \label{Fig5}
    \end{subfigure}\hfill
    \begin{subfigure}[b]{0.3\textwidth}
        \scalebox{0.75}{
            \begin{tikzpicture}
                \Vertex[x=0,y=0,label=1,color=white,style=dotted]{A}
                \Vertex[x=0,y=2,label=2,color=white]{B}
                \Vertex[x=-2,y=0,label=3,color=white]{C}
                \Vertex[x=2,y=0,label=4,color=white]{D}
                \Vertex[x=0,y=-2,label=5,color=white]{E}
                \Vertex[x=-2,y=-2,label=\textcolor{white}{6},color=black]{F}
                \Vertex[x=-2,y=2,label=7,color=white]{G}
                \Vertex[x=2,y=2,label=8,color=white]{H}
                \Vertex[x=-2,y=3.5,label=9,color=white]{I}
                \Edge[lw=2pt,style=dotted](A)(B)
                \Edge[lw=2pt,style=dotted](A)(C)
                \Edge[lw=2pt,style=dotted](A)(D)
                \Edge[lw=2pt,style=dotted](A)(E)
                \Edge[lw=2pt](C)(F)
                \Edge[lw=2pt](E)(F)
                \Edge[lw=2pt](C)(G)
                \Edge[lw=2pt,style=dotted](B)(G)
                \Edge[lw=2pt,style=dotted](B)(H)
                \Edge[lw=2pt,style=dotted](H)(D)
                \Edge[lw=2pt,style=dotted](G)(I)
            \end{tikzpicture}
        }
        \caption{E$p$M solution after imposing SPC (dispersion = $30$).} \label{Fig6}
    \end{subfigure}
    \caption{Two E$p$M solutions of the road network; note: solid lines denote edges allocated to the first territory, and dotted lines denote edges allocated to the second territory.}
    \label{fig:main}
\end{figure}
\end{footnotesize}

\begin{proof}
Fix the selected centers of any optimal E$p$M solution and construct an optimal solution $(w^*,x^*)$ as follows. Choose an ordering of the selected centers and assign each edge to a nearest selected center, breaking every equal-distance tie in favor of the center that appears first in the ordering. For the fixed set of selected centers, these assignments minimize the distance of each edge to its assigned center and therefore preserve feasibility and optimality.

\begin{figure}
    \begin{subfigure} [b]{0.45\textwidth}
	\begin{center}
		\begin{tikzpicture}
		\Vertex[x=2,y=0,color=white,label=$k$]{A}
		\Vertex[x=0,y=0,color=white,label=$j$]{B}
        \Vertex[x=4,y=-3,color=white,label=$i$]{D}		
        \Vertex[x=0.75,y=-3,color=white,label=$i^{\prime}$]{G}	
        \Vertex[x=2,y=-3,color=white,label=$l$]{E}
        \Vertex[x=2,y=-1.5,color=white,label=$m$]{F}
		\Edge[lw=2pt](A)(B)
		\Edge[style={dashed}](A)(F)
        \Edge[style={dashed}](G)(E)
        \Edge[lw=2pt](E)(F)
        \Edge[style={dashed}](E)(D)
				\end{tikzpicture}\\
    \caption{Planar graph example for Case 1.} 
    \label{case1}
	\end{center}
 
 \end{subfigure}
 \begin{subfigure} [b]{0.45\textwidth}
 	\begin{center}
		\begin{tikzpicture}
		\Vertex[x=2,y=0,color=white,label=$k$]{A}
		\Vertex[x=0,y=0,color=white,label=$j$]{B}
        \Vertex[x=4,y=-3,color=white,label=$i$]{D}		
        \Vertex[x=0.75,y=-1.5,color=white,label=$i^{\prime}$]{G}		
        \Vertex[x=2,y=-3,color=white,label=$l$]{E}
        \Vertex[x=2,y=-1.5,color=white,label=$m$]{F}
		\Edge[lw=2pt](A)(B)
		\Edge[style={dashed}](A)(F)
        \Edge[style={dashed}](G)(F)
        \Edge[lw=2pt](E)(F)
        \Edge[style={dashed}](E)(D)
		\end{tikzpicture}\\
        \caption{Planar graph example for Case 2.}
        \label{case2}
	\end{center}
 \end{subfigure}
 \caption{Planar graph examples illustrating the two cases in the proof of \Cref{thm:3}.}
 \label{illustration}
 \end{figure}
 
In the optimal solution $(w^*,x^*)$, consider an edge $(j,k)$ assigned to center $i$, and suppose that some edge of $SP_{i,(j,k)}$ is not assigned to $i$. Let $(l,m)\in SP_{i,(j,k)}$ be such an edge, which is instead assigned to a center $i'\neq i$, and let $l$ be the endpoint of $(l,m)$ reached first along $SP_{i,(j,k)}$ from $i$. Since $i'$ is a nearest selected center of $(l,m)$ and $i$ is also selected, $d_{i',(l,m)}\leq d_{i,(l,m)}$. We consider two cases, illustrated in Figure~\ref{illustration}.

\noindent\textbf{Case 1:} $SP_{i',(l,m)}$ reaches $(l,m)$ at $l$ (Figure~\ref{case1}). Since $d_{i',(l,m)}\leq d_{i,(l,m)}$, we have
\begin{subequations}
\begin{align}
d_{i,(j,k)}
&= d_{i,(l,m)}+d_{l,(j,k)}\\
&\geq d_{i',(l,m)}+d_{l,(j,k)}\\
&\geq d_{i',(j,k)}.
\end{align}
\end{subequations}
The last inequality follows because the distance from $i'$ to $(j,k)$ through $l$ provides an upper bound on $d_{i',(j,k)}$. Consequently, $d_{i,(j,k)}\geq d_{i',(j,k)}$. Two subcases arise.

\emph{Subcase 1a:} $d_{i',(l,m)}<d_{i,(l,m)}$. Then, $d_{i,(j,k)}>d_{i',(j,k)}$; therefore, reallocating $(j,k)$ from $i$ to $i'$ preserves feasibility and strictly reduces the objective value, contradicting the optimality of $(w^*,x^*)$.

\emph{Subcase 1b:} $d_{i',(l,m)}=d_{i,(l,m)}$. Then, $i$ is also a nearest selected center of $(l,m)$; therefore, the assignment of $(l,m)$ to $i'$ implies that $i'$ precedes $i$ in the chosen ordering. Furthermore, $d_{i,(j,k)}\geq d_{i',(j,k)}$ and the fact that $(j,k)$ is assigned to a nearest selected center give $d_{i,(j,k)}=d_{i',(j,k)}$. Hence, $i'$ is also a nearest selected center of $(j,k)$ and precedes $i$, contradicting the assignment of $(j,k)$ to $i$.

\noindent\textbf{Case 2:} $SP_{i',(l,m)}$ reaches $(l,m)$ at $m$ (Figure~\ref{case2}). Since $d_{i',(l,m)}\leq d_{i,(l,m)}$ and $\ell_{(l,m)}>0$, we have
\begin{subequations}
\begin{align}
d_{i,(j,k)}
&= d_{i,(l,m)}+d_{l,(j,k)}\\
&= d_{i,(l,m)}+\ell_{(l,m)}+d_{m,(j,k)}\\
&> d_{i',(l,m)}+d_{m,(j,k)}\\
&\geq d_{i',(j,k)}.
\end{align}
\end{subequations}
The last inequality follows because the distance from $i'$ to $(j,k)$ through $m$ provides an upper bound on $d_{i',(j,k)}$. Consequently, $d_{i,(j,k)}>d_{i',(j,k)}$. Reallocating $(j,k)$ from $i$ to $i'$ therefore preserves feasibility and strictly reduces the objective value, contradicting the optimality of $(w^*,x^*)$.

Every case leads to a contradiction; hence, all edges of $SP_{i,(j,k)}$ are assigned to $i$. Because $(j,k)$ and $i$ were arbitrary, $(w^*,x^*)$ satisfies all SPC constraints.
\end{proof}
Theorem~\ref{thm:3} has a broader implication for the relationship among E$p$M, EC$p$M, and EC$p$M-SPC, as stated next.
\begin{corollary}\label{cor:equal-values}
Let $G=(V,E)$ be an undirected connected graph with $|E|\geq2$ and $\ell_e>0$ for all $e\in E$. Then, E$p$M, EC$p$M, and EC$p$M-SPC have the same optimal objective value. Consequently, every optimal EC$p$M solution assigns each edge to a nearest selected center.
\end{corollary}

\begin{proof}
\emph{Equality of optimal objective values.} Since EC$p$M is obtained from E$p$M by imposing contiguity, every EC$p$M-feasible solution is also E$p$M-feasible. Moreover, the SPC constraints ensure that each assigned edge is connected to its selected center through edges of the same territory, meaning that every EC$p$M-SPC-feasible solution is also EC$p$M-feasible. Therefore, the optimal objective value of E$p$M cannot exceed that of EC$p$M, which cannot exceed that of EC$p$M-SPC. Conversely, by Theorem~\ref{thm:3}, there exists an optimal E$p$M solution satisfying all SPC constraints. This solution is feasible for both EC$p$M and EC$p$M-SPC, establishing that all three models have the same optimal objective value. In particular, imposing the SPC constraints preserves at least one common optimal solution of E$p$M and EC$p$M.

\emph{Nearest-center allocation.} Now let $(w^*,x^*)$ be any optimal EC$p$M solution and suppose that some edge $(j,k)$ is assigned to a selected center $i$ while another selected center $i'$ satisfies $d_{i',(j,k)}<d_{i,(j,k)}$. Reassigning $(j,k)$ to $i'$ gives an E$p$M-feasible solution with an objective value strictly below the common optimal value, resulting in a contradiction. Hence, every edge in an optimal EC$p$M solution is assigned to a nearest selected center.
\end{proof}\noindent
Although the nearest-center assignment established above ensures contiguity, it may not satisfy additional restrictions on territory assignments, such as capacity, workload-balance, or service requirements, making an explicit representation of contiguity particularly important in such extensions.

The results above can be extended to a more general network-distance measure than the one defined in Definition~\ref{def_3}. In particular, \citet{jost2025districting} consider an alternative node-to-edge distance that accounts for travel along a portion of the edge. For a fixed $\alpha\in[0,1]$, the generalized network distance is
\begin{equation}
\tilde d_{i,(j,k)}:=d_{i,(j,k)}+\alpha\ell_{(j,k)}, \qquad \forall i\in V,\ (j,k)\in E. \label{eq:alternative-distance}
\end{equation}
Setting $\alpha=0$ recovers $d_{i,(j,k)}$, while $\alpha=0.5$ corresponds to measuring the distance to the midpoint of edge $(j,k)$.

\begin{proposition}\label{prop:shifted-distance}
For any fixed $\alpha\in[0,1]$, let $\widetilde{\mathrm{E}p\mathrm{M}}$ and $\widetilde{\mathrm{EC}p\mathrm{M}}$ denote the $\mathrm{E}p\mathrm{M}$ and $\mathrm{EC}p\mathrm{M}$ formulations obtained by replacing $d_{i,(j,k)}$ with the generalized network distance $\tilde d_{i,(j,k)}$ defined in \eqref{eq:alternative-distance}. Then, each modified formulation has the same feasible region and the same set of optimal solutions as its original counterpart.
\end{proposition}

\begin{proof}
The distance measure affects the formulations in two ways. First, in formulations that include the SPC constraints, it determines the sets $SP_{i,(j,k)}$ for all $i\in V$ and $(j,k)\in E$. Nevertheless, the additional term $\alpha\ell_{(j,k)}$ in~\eqref{eq:alternative-distance} is independent of the center node $i$ and, because the same term is added whether edge $(j,k)$ is reached through node $j$ or node $k$, the comparison between the two endpoint distances remains unchanged. Therefore, the shortest path from center node $i$ to edge $(j,k)$, and accordingly the sets $SP_{i,(j,k)}$ and the SPC constraints, remain unchanged. Since no other constraint involves the distance measure, each original model and its modified counterpart have the same feasible region.

Second, it determines the distance values in objective~\eqref{obj}. Let $(w,x)$ be any feasible solution to $\widetilde{\mathrm{E}p\mathrm{M}}$ or $\widetilde{\mathrm{EC}p\mathrm{M}}$. By definition, the objective under the generalized network distance becomes
\[
\tilde{z}:=\sum_{i\in V}\sum_{(j,k)\in E}\tilde d_{i,(j,k)}x_{i,(j,k)}
=\sum_{i\in V}\sum_{(j,k)\in E}d_{i,(j,k)}x_{i,(j,k)}
+\alpha\sum_{(j,k)\in E}\ell_{(j,k)}\sum_{i\in V}x_{i,(j,k)}.
\]
By constraints~\eqref{const:110}, each edge is assigned to one center, implying $\sum_{i\in V}x_{i,(j,k)}=1$ for every $(j,k)\in E$. Therefore,
\[
\tilde{z}
=\sum_{i\in V}\sum_{(j,k)\in E}d_{i,(j,k)}x_{i,(j,k)}
+\alpha\sum_{(j,k)\in E}\ell_{(j,k)}.
\]
Thus, for every feasible solution, $\tilde{z}$ equals its objective value in the original model plus the constant $\alpha\sum_{(j,k)\in E}\ell_{(j,k)}$. Hence, each original model and its modified counterpart have the same set of optimal solutions.
\end{proof}

\begin{corollary}\label{cor:spc-supervalid-alpha}
For any fixed $\alpha\in[0,1]$, the SPC constraints defined under $\tilde d$ are supervalid inequalities for $\widetilde{\mathrm{E}p\mathrm{M}}$.
\end{corollary}
\begin{proof}
By Proposition~\ref{prop:shifted-distance}, $\widetilde{\mathrm{E}p\mathrm{M}}$ and $\mathrm{E}p\mathrm{M}$ have the same optimal solutions, and the sets $SP_{i,(j,k)}$ are unchanged. Hence, an optimal solution of $\mathrm{E}p\mathrm{M}$ satisfying all SPC constraints, whose existence follows from Theorem~\ref{thm:3}, is also optimal for $\widetilde{\mathrm{E}p\mathrm{M}}$. Therefore, for any fixed $\alpha\in[0,1]$, the SPC constraints are supervalid inequalities for $\widetilde{\mathrm{E}p\mathrm{M}}$, as they preserve at least one of its optimal solutions.
\end{proof}
Since each original formulation and its modified counterpart have the same feasible and optimal solutions, every property of feasible or optimal solutions established for one also holds for the other. Consequently, for ease of exposition, we set $\alpha=0$ for the remainder of the paper without loss of generality; all subsequent theoretical results apply to any fixed $\alpha\in[0,1]$.
\section{Model-Strengthening Inequalities for EC$p$M}
\label{sec:Logic_cuts}
Next, we exploit structural properties of EC$p$M solutions to derive three families of inequalities that strengthen the formulation. They differ in the feasible and optimal solution sets they preserve. We then describe simple strategies for selecting and adding these inequalities while keeping the number of additional constraints manageable.

The first two families of inequalities build on solution structures identified by \citet{kassem2023models} for the edge-based districting (EBD) problem. EBD differs from EC$p$M by imposing lower and upper bounds on the total demand assigned to each district. For the first family, we establish a sufficient condition under which the corresponding structure holds in at least one optimal EC$p$M solution; this condition addresses a special case overlooked in \citet{kassem2023models}, where the inequality may otherwise exclude all optimal solutions. The authors also introduce two constraint systems to enforce the second solution structure. We show that this structure holds for every optimal EC$p$M solution and derive one system from the other, which establishes that one of them is redundant. The third family, on the other hand, is entirely new and exploits paths in the road network whose interior vertices have degree two. Unlike the first two families, which are supervalid inequalities and non-valid logic cuts, respectively, the inequalities in the third family are valid for all integer-feasible EC$p$M solutions.
\subsection{Center-Incidence Supervalid Inequalities}
The first family of inequalities exploits the equal distance from the two endpoints of an edge to that edge, which is zero under $\alpha=0$. Let $E^*_c:=\{e\in E:x^*_{c,e}=1\}$ denote the set of edges assigned to center $c$ in an optimal EC$p$M solution. By contiguity, the center of a nonempty territory must be located at a node incident to $E_c^*$. There are at most $|E^*_c|+1$ such nodes, with equality when the subgraph formed by $E^*_c$ is a tree. The proposed inequality narrows these possible locations to nodes incident to at least two edges of $E_c^*$. 

Since these inequalities require each selected center to be incident to at least two assigned edges, we first establish conditions under which an optimal solution exists with at least two edges in every territory. We use the minimum size of a \emph{vertex cover} to establish when EC$p$M admits an optimal solution with $|E_c^*|\geq2$ for every selected center $c$. For $G=(V,E)$, a vertex cover is a set of nodes containing at least one endpoint of every edge. We denote the minimum number of nodes in a vertex cover of $G$ by $\tau(G)$ \citep{merris2011graph}.

\begin{lemma}\label{lem:two-edges}
Let $G=(V,E)$ be a connected graph with $\ell_e>0$ for all $e\in E$. The conditions $p\leq\tau(G)$ and $|E|\geq2p$ are sufficient for EC$p$M to admit an optimal solution $(w^*,x^*)$ in which every territory contains at least two allocated edges, i.e., $\lvert E^*_c\rvert\geq2$ for every $c\in V$ with $w_c^*=1$. For path graphs, these conditions are also necessary.
\end{lemma}

\begin{proof} Among all optimal EC$p$M solutions, choose one with the minimum number of singleton territories and denote it as $(w^*,x^*)$, with $W:=\{i\in V:w_i^*=1\}$ denoting the set of selected centers in this solution; we show that this minimum number is zero. The proof proceeds through the following steps.

\emph{No selected center is redundant in an optimal solution when $p\leq\tau(G)$.} Suppose removing a center leaves a set $W'\subset W$ of $p-1$ centers with the same objective value. Consider the corresponding E$p$M solution. If some edge $e=(u,v)$ has positive distance from every center in $W'$, add either $u$ or $v$ as a center. The distance from $e$ to this new center is zero, while the distance from every other edge to its nearest center cannot increase. The resulting E$p$M solution has an objective value strictly below the common optimal value of E$p$M and EC$p$M established in Corollary~\ref{cor:equal-values}, a contradiction. Therefore, every edge must have at least one endpoint in $W'$. Thus, $W'$ is a vertex cover of $G$, implying $\tau(G)\leq|W'|=p-1$, which contradicts $p\leq\tau(G)$. Hence, no selected center is redundant.

Suppose, for contradiction, that $(w^*,x^*)$ contains at least one singleton territory.

\emph{A singleton territory can be moved to an adjacent node without increasing the objective.} Let $c$ be the center of a singleton territory whose unique assigned edge is $e$. Since $E^*_c$ is contiguous with its center, its unique edge must be incident to $c$, i.e., $e=(c,u)$ for some $u\in V$. Moreover, $u$ cannot also be a center, since reassigning $e$ to $u$ at zero cost would leave $c$ with no allocated edge and therefore make it redundant.

By Corollary~\ref{cor:equal-values}, every optimal EC$p$M solution is a nearest-center allocation. Consider any edge $f=(c,w)\in\delta(c)\setminus\{e\}$. Since $d_{c,f}=0$ and $f$ is not assigned to $c$, its assigned center must also have zero distance to $f$. Positive edge lengths imply that this center is $w$. Moreover, $E^*_w$ cannot be a singleton, since reassigning $f$ to $c$ at zero cost would leave $w$ with no allocated edge and make it redundant. Hence, $|E^*_w|\geq2$.

Reassign $f$ from $w$ to $c$. This does not change the objective because $d_{c,f}=d_{w,f}=0$. Adding $f$ to $E_c$ preserves contiguity because $f$ is incident to $c$. Removing $f$ from $E_w$ also preserves contiguity; otherwise, some edge of $E_w\setminus\{f\}$ would become disconnected from $w$. Since $E_w$ is contiguous, a path within $E_w$ from $w$ to such an edge must begin with $f=(w,c)$ and then leave $c$ through another edge assigned to $w$. This edge has distance zero from the selected center $c$ and positive distance from $w$, which contradicts the nearest-center property. Hence, the reassignment preserves EC$p$M feasibility. If $E_w$ contains $q\geq3$ edges before the reassignment, the two territory sizes change from $(1,q)$ to $(2,q-1)$, producing another optimal solution with fewer singleton territories, a contradiction. Therefore, if the singleton is not eliminated, $E_w$ contains exactly two edges before the reassignment, which simply moves the singleton from $c$ to $w$.

\emph{If a singleton territory exists, some territory contains at least three edges.} By assumption, $|E|\geq2p$. Moreover, no territory is empty since otherwise its center would be redundant. If a singleton territory exists and every other territory contains at most two edges, then $|E|\leq1+2(p-1)=2p-1$, which contradicts $|E|\geq2p$. Therefore, the existence of a singleton territory implies that at least one other territory contains three or more edges.

\emph{Any singleton territory can be moved until it is eliminated.} Fix a simple path $R=(z_0,z_1,\ldots,z_t)$ from the initial singleton center $z_0$ to the center $z_t$ of a territory containing at least three edges. The local reassignment and center-relocation steps preserve EC$p$M feasibility and the objective value, so they may be repeated along the path $R$. Each resulting solution is therefore optimal and, by Corollary~\ref{cor:equal-values}, is a nearest-center allocation. The nonredundancy argument above also applies to each such solution.

After $r$ moves, let the singleton be centered at $z_r$ with unique assigned edge $g$, and let $h=(z_r,z_{r+1})$ be the next edge of $R$. Since the current singleton territory is contiguous, $g$ is incident to $z_r$. If $h=g$, then $z_{r+1}$ cannot be a center, since otherwise reassigning $g$ to $z_{r+1}$ at zero cost would make $z_r$ redundant. We may therefore relocate the singleton center from $z_r$ to $z_{r+1}$ while keeping $g$ assigned to it; this preserves EC$p$M feasibility and the objective value, so the singleton moves one step along $R$. If $h\neq g$, then $h$ is not assigned to $z_r$. Since $d_{z_r,h}=0$ and the allocation is a nearest-center allocation, the center to which $h$ is assigned must also have zero distance to $h$; positive edge lengths imply that this center is $z_{r+1}$. Reassigning $h$ to $z_r$ preserves EC$p$M feasibility and the objective value by the argument above. If the territory centered at $z_{r+1}$ contains at least three edges before the reassignment, the singleton is eliminated; otherwise, it contains exactly two edges, and the singleton moves to $z_{r+1}$.

Thus, if the singleton is not eliminated at an earlier step, each operation moves it one node closer to $z_t$ along the fixed path $R$. The territory centered at $z_t$ remains unchanged until the singleton reaches it. Therefore, after at most $t\leq|V|-1$ operations, the singleton reaches this territory and is eliminated. The resulting solution is optimal and has fewer singleton territories than $(w^*,x^*)$, which contradicts its choice. Hence, $(w^*,x^*)$ has no singleton territories, i.e., $|E_c^*|\geq2$ for every selected center $c$.

\emph{The conditions are necessary for path graphs.} If every selected territory contains at least two edges, then, since the $p$ territories partition $E$, $|E|=\sum_{c:w_c^*=1}|E_c^*|\geq2p$. For a path graph, $\tau(G)=\lfloor(|E|+1)/2\rfloor$, so $|E|\geq2p$ also implies $p\leq\tau(G)$. Thus, both conditions are necessary for path graphs. \end{proof}
\begin{theorem}\label{thm:thmvi1}
Let $G=(V,E)$ be a connected graph with $\ell_e>0$ for all $e\in E$, and let $\delta(i)$ denote the set of edges incident to node $i$. If EC$p$M admits an optimal solution in which every territory associated with a selected center contains at least two edges, then it admits an optimal solution $(w^*,x^*)$ such that
\begin{equation}\label{eq:VI1}
\sum_{e\in\delta(i)}x^*_{i,e}\geq 2w^*_i,\qquad \forall i\in V.
\end{equation}
In particular, this conclusion holds when $p\leq\tau(G)$ and $|E|\geq2p$, by Lemma~\ref{lem:two-edges}.
\end{theorem}
\begin{proof}
By assumption, there exists an optimal solution $(w^*,x^*)$ such that $|E_c^*|\geq2$ for every selected center $c$. We establish \eqref{eq:VI1} for this solution. If $w_i^*=0$, the inequality follows immediately. Thus, consider $w_i^*=1$.

Suppose, for contradiction, that some selected center $c\in V$ has fewer than two incident edges assigned to its territory, i.e., $\sum_{e\in\delta(c)}x_{c,e}^*<2$. Since $E_c^*$ is nonempty and contiguous with $c$, at least one edge incident to $c$ must belong to $E_c^*$, so $\sum_{e\in\delta(c)}x_{c,e}^*\geq1$. Together with the assumption above and the integrality of $x$, this gives $\sum_{e\in\delta(c)}x_{c,e}^*=1$. Let $(c,j)$ denote this unique incident edge. Since $|E_c^*|\geq2$, $E_c^*$ contains at least one other edge.

As $(c,j)$ is the only edge of $E_c^*$ incident to $c$, the nearest-center property implies that every other edge incident to $c$ is assigned to its other endpoint, which must therefore be a selected center. Every shortest path from $c$ to an edge $e\in E_c^*\setminus\{(c,j)\}$ must consequently begin with $(c,j)$; otherwise, it would pass through another selected center strictly closer to $e$. Hence,
\[
d_{c,e}=d_{j,e}+\ell_{(c,j)},\qquad \forall e\in E_c^*\setminus\{(c,j)\},
\]
while $d_{c,(c,j)}=d_{j,(c,j)}=0$. Since $\ell_{(c,j)}>0$, we have $d_{j,e}<d_{c,e}$ for every $e\in E_c^*\setminus\{(c,j)\}$.

We first show that $j$ cannot be a selected center. Suppose $w_j^*=1$. Reassigning all edges of $E_c^*$ from $c$ to $j$ preserves contiguity and strictly decreases the objective value. Edge $(c,j)$ has zero distance to both centers, while every other edge in $E_c^*$ is closer to $j$ than to $c$. This contradicts the optimality of $(w^*,x^*)$. Hence, $w_j^*=0$.

We can therefore construct another solution by relocating the center of $E_c^*$ from $c$ to $j$. Set $w_c=0$ and $w_j=1$, assign all edges of $E_c^*$ to $j$, and leave all other assignments unchanged. The resulting change in the objective value is
\[
\begin{aligned}
\Delta
&=\sum_{e\in E_c^*}d_{j,e}-\sum_{e\in E_c^*}d_{c,e}\\
&=\left(d_{j,(c,j)}-d_{c,(c,j)}\right)+\sum_{e\in E_c^*\setminus\{(c,j)\}}\left(d_{j,e}-d_{c,e}\right)\\
&=-\ell_{(c,j)}\left(|E_c^*|-1\right).
\end{aligned}
\]
Because $\ell_{(c,j)}>0$ and $|E_c^*|\geq2$, we have $\Delta<0$, again contradicting optimality.

Thus, every selected center in $(w^*,x^*)$ is incident to at least two edges assigned to its territory, and \eqref{eq:VI1} holds.
\end{proof}
\begin{remark}
Under the assumptions of Theorem~\ref{thm:thmvi1}, the center-incidence inequalities~\eqref{eq:VI1} and the SPC constraints preserve a common optimal solution. To see this, take an optimal solution satisfying \eqref{eq:VI1} and fix its selected centers. For each edge with neither endpoint selected as a center, assign it to a nearest selected center, resolving equal-distance ties using a fixed ordering of the selected centers, as in Theorem~\ref{thm:3}. For each edge with at least one endpoint selected as a center, retain its original assignment.

The retained assignments are consistent with SPC. An edge with exactly one endpoint selected as a center has that endpoint as its unique nearest center. An edge whose endpoints are both selected centers requires no connecting path to its assigned center. Moreover, such an edge cannot lie on a shortest path from a selected center to another edge assigned to it, since that path would pass through another selected center strictly closer to the assigned edge. Thus, the resulting solution remains optimal, preserves \eqref{eq:VI1}, and satisfies all SPC constraints.
\end{remark}

The center-incidence supervalid inequalities~\eqref{eq:VI1} yield one inequality for each $i\in V$. Hence, this family contains only $|V|$ inequalities and can be added to the formulation up front whenever the sufficient conditions in Lemma~\ref{lem:two-edges} hold. To verify these conditions, we compute the minimum vertex-cover size once for each road network. Let $\eta_i\in\{0,1\}$ indicate whether vertex $i\in V$ belongs to the vertex cover. We solve
\begin{equation}\label{eq:vertexcover}
    \tau(G)=\min\left\{\sum_{i\in V}\eta_i:\ \eta_j+\eta_k\geq1\ \forall (j,k)\in E,\quad \eta_i\in\{0,1\}\ \forall i\in V\right\},
\end{equation}
which minimizes the number of vertices needed to cover all edges. For a given value of $p$, we then check whether $p\leq\lfloor |E|/2\rfloor$ and $p\leq\tau(G)$. When both conditions hold, inequalities~\eqref{eq:VI1} preserve at least one optimal solution.

\subsection{Center-Incidence Non-valid Logic Cuts}
Center-incidence logic cuts capture a structure of optimal EC$p$M solutions in which, if a node is selected as a center and none of its neighbors is selected, all edges incident to that center are assigned to its territory. This structure is not guaranteed to hold in the edge-based districting problem, because the workload-balance constraints may require some edges incident to the center to be assigned to other districts. Thus, the corresponding inequality proposed for EBD is heuristic and may exclude optimal solutions. In contrast, EC$p$M does not impose workload-balance constraints. We establish that this structure holds in every optimal EC$p$M solution and derive the center-incidence non-valid logic cuts.
\begin{theorem}
Let $G=(V,E)$ be a connected graph with $\ell_e>0$ for all $e\in E$. In every optimal EC$p$M solution $(w^*,x^*)$, if $i$ is a center and none of its neighbors is a center, then every edge incident to $i$ is assigned to its territory. Let $N(i)$ denote the set of nodes adjacent to node $i$. This property is enforced by inequalities
\begin{equation}\label{eq:VI2}
     \sum_{e\in\delta(i)} x^*_{i,e} \geq \lvert \delta(i)\rvert \left(w_i^*-\sum_{j\in N(i)}w_j^*\right), \qquad \forall i\in V.
\end{equation}
\end{theorem}
\begin{proof}
\emph{Structural implication.} We first establish that if a center $i$ has no neighboring center, then every edge incident to $i$ is allocated to $i$. Let $(w^*,x^*)$ be an optimal solution with $w_i^*=1$ and $\sum_{j\in N(i)} w_j^*=0$, where $N(i)$ denotes the set of nodes adjacent to $i$ in $G$. Suppose, for contradiction, that some edge $(i,a)\in\delta(i)$ is allocated to a center $k\neq i$. Since $a\in N(i)$ is not a center, $k\neq a$, so $d_{k,(i,a)}>0=d_{i,(i,a)}$, and assigning $(i,a)$ to $i$ would strictly decrease the objective. By Corollary~\ref{cor:equal-values}, every edge is assigned to a nearest selected center. However, $d_{i,(i,a)}=0<d_{k,(i,a)}$, contradicting the assignment of $(i,a)$ to $k$. Therefore, $\sum_{e\in\delta(i)}x^*_{i,e}=|\delta(i)|$.

\emph{Derivation.} For each $i\in V$, define a binary auxiliary variable $z_i^*$ that equals $1$ if and only if no node in $N(i)$ is selected as a center. This variable satisfies
\begin{equation}\label{cut2_s1}
z^*_i\geq1-\sum_{j\in N(i)}w^*_j\,.
\end{equation}
By the structure established above,
\begin{equation}\label{cut2_s2}
\sum_{e\in\delta(i)}x^*_{i,e}\geq|\delta(i)|w^*_i z^*_i.
\end{equation}
Since $w_i^*$ and $z_i^*$ are binary, the McCormick lower bound gives $w_i^*z_i^*\geq w_i^*+z_i^*-1$. Because its coefficient $|\delta(i)|$ is positive, substituting this lower bound for the bilinear term preserves validity and gives
\begin{equation}\label{cut2_s3}
\sum_{e\in\delta(i)}x^*_{i,e}\geq|\delta(i)|(w^*_i+z^*_i-1).
\end{equation}
Similarly, since $z_i^*$ has a positive coefficient on the right-hand side of \eqref{cut2_s3}, substituting its lower bound from \eqref{cut2_s1} preserves validity and yields \eqref{eq:VI2}.
\end{proof}
Similar to the center-incidence supervalid inequalities~\eqref{eq:VI1}, the center-incidence non-valid logic cuts~\eqref{eq:VI2} yield one inequality for each $i\in V$. Hence, this family also contains only $|V|$ inequalities, all of which can be added to the formulation up front.
\subsection{Path-based Valid Inequalities}
\label{fbvi}
Path-based valid inequalities are inspired by similarly named inequalities developed for the power transmission expansion planning problem in \citet{JabbariMarand2026}. In that setting, parallel transmission paths between two nodes are used to derive inequalities that characterize limits on the power that can flow through a path. Here, paths serve a different purpose, as we identify conditions under which all edges of a path must be assigned to the same territory based on the assignments of other connected subgraphs. In particular, path-based inequalities exploit paths whose interior vertices all have degree two; we refer to such paths as \emph{eligible}. Let $P$ be an eligible path with endpoints $u$ and $v$, and let $B:=\bigl(\delta(u)\cup\delta(v)\bigr)\setminus P$ denote the edges incident to $u$ or $v$ that do not belong to $P$. Figure~\ref{fig:vi3_illustration} illustrates this configuration, where $P=\langle e_1,e_2,e_3\rangle$, $V(P)=\{u,c,d,v\}$ is the vertex set of $P$, and $B=\{e_4,e_5,e_6,e_7,e_8,e_9,e_{10},e_{11}\}$.

The key structural insight is that every interior vertex of $P$ has degree two and is therefore incident only to the two edges of $P$ adjacent to it. Consequently, any path from a center outside $P$ to an edge of $P$ must pass through $u$ or $v$. If all edges in $B$ are assigned to a center $i$ and no vertex in $V(P)\setminus\{i\}$ is selected as a center, contiguity requires every edge of $P$ to be assigned to $i$.

\begin{theorem}
Let $G=(V,E)$ be a connected graph, $P$ an eligible path, and $B$ the set of edges incident to the endpoints of $P$ that do not belong to $P$. Then
\begin{equation}\label{eq:VI3}
x_{i,e}\geq\sum_{e'\in B}x_{i,e'}-|B|+1-\sum_{j\in V(P)\setminus\{i\}}w_j,\qquad \forall e\in P,\ \forall i\in V,
\end{equation}
are valid inequalities for EC$p$M. If $|E|\geq2$ and $\ell_e>0$ for all $e\in E$, they are also supervalid inequalities for E$p$M.
\end{theorem}
\begin{figure}
\centering
\begin{tikzpicture}[scale=1.15]

\Vertex[x=0,y=0,color=white,label=$u$]{U}
\Vertex[x=1.4,y=0,color=white,label=$c$]{C}
\Vertex[x=2.8,y=0,color=white,label=$d$]{D}
\Vertex[x=4.2,y=0,color=white,label=$v$]{V}

\Vertex[x=2.1,y=1.8,color=white,label=$a$]{A}
\Vertex[x=2.1,y=-1.8,color=white,label=$b$]{B}

\Vertex[x=-1.1,y=1.0,color=white,label={}]{UL}
\Vertex[x=-1.0,y=-1.0,color=white,label={}]{UD}
\Vertex[x=5.3,y=1.0,color=white,label={}]{VR}
\Vertex[x=5.2,y=-1.0,color=white,label={}]{VD}

\Edge[lw=2pt](U)(C)
\Edge[lw=2pt](C)(D)
\Edge[lw=2pt](D)(V)

\Edge[lw=2pt](U)(A)
\Edge[lw=2pt](A)(V)
\Edge[lw=2pt](U)(B)
\Edge[lw=2pt](B)(V)

\Edge[lw=2pt](U)(UL)
\Edge[lw=2pt](U)(UD)
\Edge[lw=2pt](V)(VR)
\Edge[lw=2pt](V)(VD)

\node[font=\small, above]       at (0.7,0)     {$e_1$};
\node[font=\small, above]       at (2.1,0)     {$e_2$};
\node[font=\small, above]       at (3.5,0)     {$e_3$};
\node[font=\small, above left]  at (1.05,0.9)  {$e_4$};
\node[font=\small, above right] at (3.15,0.9)  {$e_5$};
\node[font=\small, below left]  at (1.05,-0.9) {$e_6$};
\node[font=\small, below right] at (3.15,-0.9) {$e_7$};
\node[font=\small, above left]  at (-0.25,0.6) {$e_8$};
\node[font=\small, below left]  at (-0.25,-0.6) {$e_9$};
\node[font=\small, above right] at (4.35,0.5)  {$e_{10}$};
\node[font=\small, below right] at (4.35,-0.5)  {$e_{11}$};

\end{tikzpicture}
\caption{Illustration of an eligible path $P=\langle e_1,e_2,e_3\rangle$ used to generate a path-based valid inequality, with endpoints $u$ and $v$ and interior vertices of degree two.}
\label{fig:vi3_illustration}
\end{figure}
\begin{proof}
\emph{Structural implication.} Fix an eligible path $P$ with endpoints $u$ and $v$, and consider an integer-feasible EC$p$M solution $(w,x)$ in which, for some center $i\in V$, every edge of $B$ is assigned to $i$ and no vertex in $V(P)\setminus\{i\}$ is a center. We show that every edge of $P$ must also be assigned to $i$.

Suppose, for contradiction, that some edge $e\in P$ is assigned to a center $k\neq i$, i.e., $e\in E_k$. Since no vertex in $V(P)\setminus\{i\}$ is a center, we have $k\notin V(P)$. Because the interior vertices of $P$ have degree two, any path from $k$ to an edge of $P$ must reach $P$ through $u$ or $v$ and therefore contain an edge of $B$. By contiguity of $E_k$, there exists a path within $E_k$ from $k$ to $e$. This path must therefore contain an edge of $B$ assigned to $k$, which contradicts the assumption that every edge of $B$ is assigned to $i$. Hence, every edge of $P$ is assigned to $i$.

\emph{Derivation.} Fix $i\in V$ and $e\in P$, and consider any integer-feasible EC$p$M solution. Define binary auxiliary variables $y_i$ and $z_i$ as follows: $y_i=1$ if and only if every edge of $B$ is assigned to $i$, and $z_i=1$ if and only if no vertex in $V(P)\setminus\{i\}$ is a center. These variables satisfy
\begin{subequations}
\begin{align}\label{eq:b1}
    & y_i\geq \sum_{e'\in B}x_{i,e'}-|B|+1,\\
    & z_i\geq 1-\sum_{j\in V(P)\setminus\{i\}}w_j.\label{eq:b2}
\end{align}
\end{subequations}
Thus, $y_i z_i=1$ means that every edge of $B$ is assigned to $i$ and no vertex in $V(P)\setminus\{i\}$ is a center. The solution structure established above therefore implies $x_{i,e}=1$, and hence $x_{i,e}\geq y_i z_i$. Since $y_i$ and $z_i$ are binary, the McCormick lower bound gives $y_i z_i\geq y_i+z_i-1$, which yields
\begin{equation}\label{eq:b3}
    x_{i,e}\geq y_i+z_i-1.
\end{equation}
Since $y_i$ and $z_i$ have positive coefficients on the right-hand side of \eqref{eq:b3}, substituting their lower bounds from \eqref{eq:b1} and \eqref{eq:b2} preserves validity and yields \eqref{eq:VI3}.

Finally, if $|E|\geq2$ and $\ell_e>0$ for all $e\in E$, Corollary~\ref{cor:equal-values} establishes that E$p$M and EC$p$M have the same optimal objective value. Consequently, E$p$M admits an optimal solution that is feasible for EC$p$M and therefore satisfies \eqref{eq:VI3}. Hence, path-based inequality~\eqref{eq:VI3} is also a supervalid inequality for E$p$M.
\end{proof}

\subsection{Selection Strategy for Path-based Valid Inequalities}\label{impandsel}
Unlike inequalities~\eqref{eq:VI1} and~\eqref{eq:VI2}, path-based valid inequalities~\eqref{eq:VI3} can be generated for every $i\in V$ and every edge of any simple path whose interior vertices have degree two in $G$. An eligible path with $m$ edges contains $\binom{m+1}{2}$ contiguous eligible subpaths. If the inequalities are generated for all such subpaths, that path alone yields $|V|\binom{m+2}{3}$ inequalities. Because the number of generated inequalities grows with both the number and length of eligible paths, exhaustive generation quickly becomes impractical. This motivates a path-selection rule that focuses on shorter paths located in promising, structurally relevant regions of the network, as described next.

We use the \emph{faces} of a fixed plane embedding of $G$ to identify promising paths for generating path-based inequalities~\eqref{eq:VI3}. A face is a connected component of $\mathbb{R}^2\setminus G$, corresponding to a connected region of the plane bounded by the embedded graph, with one unbounded outer face \citep{diestel2025}. We represent each bounded face by its boundary edge set and let $F_1$ and $F_2$ denote the edge sets of two such faces. Suppose that $P_0:=F_1\cap F_2$ forms a simple path with endpoints $u$ and $v$. The shared path $P_0$ and the complementary boundary paths $F_1\setminus P_0$ and $F_2\setminus P_0$ provide three $u$--$v$ connections. For example, in Figure~\ref{fig:vi3_illustration}, the two bounded faces have boundary edge sets $F_1=\{e_1,e_2,e_3,e_4,e_5\}$ and $F_2=\{e_1,e_2,e_3,e_6,e_7\}$. Hence, $P_0=\langle e_1,e_2,e_3\rangle$ is their shared path, while $\{e_4,e_5\}$ and $\{e_6,e_7\}$ form the two complementary $u$--$v$ paths. We treat each of these three paths as a \emph{candidate corridor} and retain only the eligible ones.

Cycles other than facial boundaries could also be used to construct alternative connections between the endpoints of a candidate corridor. However, facial boundaries enclose a single face of the fixed plane embedding, whereas a general cycle may enclose several faces and therefore extend over a larger portion of the network. Consequently, adjacent faces provide local alternative connections around their shared corridor. This locality is important for path-based inequalities~\eqref{eq:VI3}, since the selection rule favors corridors for which a short alternative route is available. More broadly, restricting attention to local corridors focuses the generated inequalities on configurations that are more relevant to compact territorial assignments, where nearby edges tend to belong to the same territory.

The selection rule follows from the relationship between inequalities~\eqref{eq:VI3} and the SPC constraints. For any $e'\in B$, the right-hand side of~\eqref{eq:VI3} is at most $x_{i,e'}$, since $x_{i,\hat e}-1\leq0$ for every $\hat e\in B\setminus\{e'\}$ and $-w_j\leq0$ for every $j\in V(P)\setminus\{i\}$. The SPC constraints also imply $x_{i,e}\geq x_{i,e'}$ whenever $e\in SP_{i,e'}$. Therefore, if the shortest path $SP_{i,e'}$ used in the SPC constraints contains $e$, the corresponding inequality~\eqref{eq:VI3} is implied by the SPC constraints and does not strengthen the relaxation. Accordingly, we retain only candidate paths that are strictly longer than the shortest among the three local connections. This comparison serves as a selection heuristic since it uses only the three connections associated with the face pair. Specifically, it only screens candidate paths and does not prove whether an inequality is redundant with respect to the full set of SPC constraints.

We further control the number of generated path-based valid inequalities by restricting the sizes of both $B$ and the face boundaries considered. Since each edge $e'\in B$ contributes the nonpositive term $(x_{i,e'}-1)$ to the right-hand side of inequality~\eqref{eq:VI3}, larger values of $|B|$ tend to weaken the inequality. We therefore retain only corridors with $|B|\leq\beta$, where $\beta$ is a small value selected for each road network based on the number of resulting cuts. In addition, we impose a cap on the number of boundary edges per face to keep the alternative connections local and limit the number of candidate corridors.
\section{Computational Test Results}\label{sec:computational_results}
Computational experiments evaluate the effects of the SPC constraints and the proposed inequalities on solving EC$p$M and E$p$M. All models are implemented in Python 3.10.18 and solved with CPLEX 22.1.1. The experiments are run on an HPC node equipped with two 32-core Intel Xeon Platinum 8462Y+ processors, with 32 CPU cores requested for each job. A 12-hour time limit is imposed for each instance and solution method.

The test instances are derived from 14 real-world road networks in Denmark originally developed for the capacitated arc routing problem \citep{kiilerich2018new}. These networks, used for waste-collection applications, are repurposed here as general location-allocation instances for evaluating the E$p$M and EC$p$M models. The networks vary in size, containing between $198$ and $3{,}462$ nodes and between $265$ and $3{,}809$ edges. Table~\ref{tab:cut3_generation} reports the dimensions of each road network (RN), the minimum vertex-cover size $\tau(G)$ obtained by solving model~\eqref{eq:vertexcover}, and the bound $\lfloor |E|/2\rfloor$. The latter two quantities are used to verify the center-incidence conditions in Lemma~\ref{lem:two-edges}. For each road network, six values of $p$ are considered, namely, $p\in\{2,10,30,40,50,100\}$, resulting in a total of $84$ instances. Since $p\leq\lfloor |E|/2\rfloor$ and $p\leq\tau(G)$ for every instance, all instances satisfy the sufficient conditions in Lemma~\ref{lem:two-edges}. Table~\ref{tab:cut3_generation} also reports statistics for the path-selection procedure used to generate the path-based valid inequalities, as discussed below. The final column reports the total preprocessing time required to compute $\tau(G)$ and identify faces and eligible paths for the path-based inequalities.

\begin{center}
\captionof{table}{Network dimensions, center-incidence bounds, path-selection and generation statistics, and preprocessing times.}
\label{tab:cut3_generation}
\small
\setlength{\tabcolsep}{4pt}
\renewcommand{\arraystretch}{0.9}
\begin{tabular}{@{}lr@{\hspace{12pt}}rc@{\hspace{12pt}}rrrrr@{\hspace{12pt}}r@{}}
\toprule
\multicolumn{2}{c}{Network} & \multicolumn{2}{c}{Center-incidence bounds} & \multicolumn{5}{c}{Path-based inequalities} & \multicolumn{1}{c}{Total preprocessing} \\
\cmidrule(lr){1-2}\cmidrule(lr){3-4}\cmidrule(lr){5-9}\cmidrule(lr){10-10}
RN & $(|V|,|E|)$ & $\tau(G)$ & $\lfloor |E|/2\rfloor$ & $|F|$ & $|\mathrm{Cand.}|$ & $|\mathrm{Cand.}_{|B|\leq\beta}|$ & ND & $|\mathrm{Cuts}|$ & Time (s) \\
\midrule
F15\_g & $(198,265)$ & 102 & 132 & 63 & 107 & 67 & 1 & 198 & 0.35 \\
N17\_g & $(268,305)$ & 130 & 152 & 21 & 14 & 1 & 1 & 536 & 0.08 \\
K13\_g & $(393,421)$ & 184 & 210 & 10 & 2 & 1 & 1 & 786 & 0.08 \\
F6\_p & $(502,741)$ & 273 & 370 & 232 & 476 & 3 & 3 & 3,012 & 0.34 \\
O12\_g & $(761,852)$ & 347 & 426 & 64 & 37 & 1 & 1 & 1,522 & 0.12 \\
S9\_p & $(782,1{,}030)$ & 421 & 515 & 205 & 286 & 3 & 2 & 3,128 & 0.27 \\
N16\_g & $(915,1{,}030)$ & 434 & 515 & 69 & 53 & 1 & 1 & 1,830 & 0.14 \\
S9\_g & $(1{,}094,1{,}447)$ & 590 & 723 & 286 & 360 & 3 & 2 & 4,376 & 0.29 \\
S11\_g & $(1{,}564,1{,}805)$ & 761 & 902 & 148 & 107 & 2 & 2 & 6,256 & 0.19 \\
N11\_g & $(2{,}134,2{,}411)$ & 1,028 & 1,205 & 149 & 93 & 2 & 2 & 8,536 & 0.20 \\
N9\_g & $(2{,}018,2{,}655)$ & 1,065 & 1,327 & 514 & 752 & 6 & 5 & 22,198 & 0.74 \\
N15\_g & $(2{,}171,2{,}459)$ & 1,039 & 1,229 & 175 & 152 & 1 & 1 & 4,342 & 0.29 \\
K14\_g & $(3{,}462,3{,}809)$ & 1,625 & 1,904 & 168 & 128 & 2 & 2 & 17,310 & 0.24 \\
S14\_g & $(3{,}072,3{,}557)$ & 1,480 & 1,778 & 324 & 243 & 13 & 12 & 79,872 & 0.43 \\
\bottomrule
\end{tabular}
\end{center}

Before presenting the computational results, recall that families~\eqref{eq:VI1} and~\eqref{eq:VI2} each contain only $|V|$ inequalities. Therefore, all inequalities in either family are added up front when that family is tested. In contrast, the path-based inequalities~\eqref{eq:VI3} are generated only for paths retained by the selection procedure described in Section~\ref{impandsel}. We restrict attention to bounded faces with at most $10$ boundary edges to keep the alternative connections local, limit candidate generation, and focus on short paths. Table~\ref{tab:cut3_generation} reports the numbers of these faces ($|F|$), eligible candidate paths ($|\mathrm{Cand.}|$), candidates satisfying the size cap ($|\mathrm{Cand.}_{|B|\leq\beta}|$), \emph{non-dominated paths} (ND), and resulting inequalities ($|\mathrm{Cuts}|$). Here, non-dominated paths are candidate paths strictly longer than at least one of the other two paths connecting the same endpoints within the same face pair.

The selection procedure sharply reduces the number of paths used to generate the path-based inequalities. The restriction on $|B|$ performs most of the filtering for nearly all networks, while the dominance test removes at most one additional path in every case except for F15\_g, where it reduces the number of paths from 67 to one. We use $\beta=3$ for all networks except for F15\_g; $\beta=4$ is used for the latter because $\beta=3$ generates no inequalities. The final number of inequalities depends on both the network size and the lengths of the retained paths.

Table~\ref{tab:cuts} reports the average solution time of EC$p$M with SPC constraints (EC$p$M-SPC), solved directly with CPLEX, both without additional inequalities and with all tested combinations of the inequalities introduced in Section~\ref{sec:Logic_cuts}. In this table, $\mathrm{C}_1$ denotes the center-incidence supervalid inequalities, $\mathrm{C}_2$ the center-incidence non-valid logic cuts, and $\mathrm{C}_3$ the path-based valid inequalities. All configurations solve all $14\times 6$ instances. Adding $\mathrm{C}_2$ provides only marginal improvements in performance across the tested combinations and increases the solution time when added to the combination of $\mathrm{C}_1$ and $\mathrm{C}_3$. Overall, $\mathrm{C}_1$ and $\mathrm{C}_3$, when used together, provide the best performance. We therefore evaluate the computational effect of adding this combination (i.e., $\mathrm{C}_1\oplus\mathrm{C}_3$) in the subsequent analysis.\vspace{3pt}

\begin{center}
\captionof{table}{Computational performance of EC$p$M-SPC with different subsets of the proposed inequalities.}

\label{tab:cuts}

\small
\setlength{\tabcolsep}{4pt}
\renewcommand{\arraystretch}{0.9}

\resizebox{\textwidth}{!}{
\begin{tabular}{lrrrrrrrr}

\toprule

\multirow{2}{*}{RN}
& \multicolumn{8}{c}{Average solution time (s)} \\

\cmidrule(lr){2-9}

& SPC
& SPC $\oplus\;\text{C}_1$
& SPC $\oplus\;\text{C}_2$
& SPC $\oplus\;\text{C}_3$
& SPC $\oplus\;\text{C}_1\oplus\text{C}_2$
& SPC $\oplus\;\text{C}_1\oplus\text{C}_3$
& SPC $\oplus\;\text{C}_2\oplus\text{C}_3$
& SPC $\oplus\;\text{C}_1\oplus\text{C}_2\oplus\text{C}_3$ \\

\midrule

F15\_g & 5 & 5 & 5 & 5 & 6 & 6 & 5 & 7 \\

N17\_g & 6 & 9 & 6 & 6 & 10 & 10 & 7 & 11 \\

K13\_g & 14 & 17 & 14 & 15 & 17 & 19 & 16 & 19 \\

F6\_p & 44 & 51 & 44 & 43 & 52 & 51 & 44 & 52 \\

O12\_g & 80 & 57 & 78 & 81 & 182 & 58 & 81 & 195 \\

S9\_p & 125 & 108 & 122 & 127 & 108 & 108 & 122 & 110 \\

N16\_g & 152 & 123 & 153 & 158 & 151 & 130 & 158 & 169 \\

S9\_g & 361 & 305 & 351 & 380 & 307 & 334 & 372 & 334 \\

S11\_g & 1,072 & 637 & 1,090 & 1,164 & 622 & 650 & 1,163 & 686 \\

N11\_g & 2,082 & 2,147 & 2,025 & 2,198 & 2,254 & 2,309 & 2,142 & 2,371 \\

N9\_g & 2,451 & 2,185 & 2,551 & 2,442 & 2,262 & 2,170 & 2,555 & 2,289 \\

N15\_g & 4,044 & 3,826 & 4,029 & 2,237 & 3,895 & 1,986 & 2,370 & 2,326 \\

K14\_g & 9,568 & 6,818 & 9,259 & 10,214 & 6,943 & 6,292 & 9,828 & 7,214 \\

S14\_g & 11,680 & 8,802 & 10,875 & 9,880 & 7,936 & 6,566 & 9,430 & 6,862 \\

\midrule

Average & 2,263 & 1,792 & 2,186 & 2,068 & 1,767 & 1,478 & 2,021 & 1,618 \\

\bottomrule

\end{tabular}
}
\end{center}

Table~\ref{tab:Instances} summarizes the performance of the three featured formulations: E$p$M, EC$p$M with SPC constraints (EC$p$M-SPC), and EC$p$M with cut-set contiguity constraints (EC$p$M-CSC). E$p$M and EC$p$M-SPC are solved directly with CPLEX, while EC$p$M-CSC is solved through branch-and-cut. For each configuration, we report the average solution time over the instances solved within the time limit (Time (s)), the number of instances that reach the time limit without proving optimality (\# TL), and the average optimality gap among the time-limited instances (Gap). In the summary row, solution times are averaged across the network-level means for networks with at least one instance solved within the time limit; \# TL reports the total number of time-limited instances, and Gap reports their average optimality gap. As discussed in Section~\ref{derive}, EC$p$M-SPC is obtained by adding the SPC constraints to E$p$M. Although these constraints enforce contiguity in EC$p$M, they also serve as supervalid inequalities for E$p$M. Thus, the comparison also evaluates the computational effect of the SPC constraints on solving E$p$M. Finally, the shortest-path computations required to construct the objective coefficients (i.e., distance values) for all formulations are excluded from the reported solution times. The additional preprocessing required to solve the minimum vertex-cover problem~\eqref{eq:vertexcover} for the center-incidence supervalid inequalities and to identify faces and eligible paths for the path-based inequalities~\eqref{eq:VI3} takes less than $1$ second for every network and is therefore also excluded from the reported solution times.

Because Time (s) in Table~\ref{tab:Instances} includes only instances solved within the time limit, the solution-time results should be considered together with the number of time-limited instances. The most notable result in the table is the consistent computational advantage of EC$p$M-SPC over EC$p$M-CSC. For every network on which EC$p$M-CSC solves at least one instance, EC$p$M-SPC has a lower average solution time. In the summary row, the average solution time is $5{,}756$ seconds for EC$p$M-CSC and $2{,}263$ seconds for EC$p$M-SPC, an improvement factor of $2.5$. More importantly, EC$p$M-SPC solves all $84$ instances, whereas EC$p$M-CSC reaches the time limit on eight instances, with an average optimality gap of $58.6\%$. The difference is particularly pronounced for network K14\_g, where EC$p$M-SPC solves all six instances, while EC$p$M-CSC reaches the time limit on all six. These results show that including the SPC constraints up front is considerably more effective than dynamically separating the CSC constraints, particularly for the larger road networks.

\begin{center}
\begin{minipage}{\textwidth}
\centering
\captionof{table}{Computational performance of E$p$M, EC$p$M-CSC, and EC$p$M-SPC, with and without the cuts.}
\label{tab:Instances}
\small
\setlength{\tabcolsep}{3pt}
\renewcommand{\arraystretch}{0.9}
\resizebox{\textwidth}{!}{
\begin{tabular}{l*{9}{r}*{9}{r}}
\toprule
\multirow{2}{*}{RN}
& \multicolumn{3}{c}{E$p$M}
& \multicolumn{3}{c}{EC$p$M-CSC}
& \multicolumn{3}{c}{EC$p$M-SPC}
& \multicolumn{3}{c}{E$p$M $\oplus$ Cuts}
& \multicolumn{3}{c}{EC$p$M-CSC $\oplus$ Cuts}
& \multicolumn{3}{c}{EC$p$M-SPC $\oplus$ Cuts} \\
\cmidrule(lr){2-4}
\cmidrule(lr){5-7}
\cmidrule(lr){8-10}
\cmidrule(lr){11-13}
\cmidrule(lr){14-16}
\cmidrule(lr){17-19}
& Time\,(s) & \# TL & Gap\,(\%)
& Time\,(s) & \# TL & Gap\,(\%)
& Time\,(s) & \# TL & Gap\,(\%)
& Time\,(s) & \# TL & Gap\,(\%)
& Time\,(s) & \# TL & Gap\,(\%)
& Time\,(s) & \# TL & Gap\,(\%) \\
\midrule
F15\_g
& 4 & 0\scriptsize{/6} & --
& 72 & 0\scriptsize{/6} & --
& 5 & 0\scriptsize{/6} & --
& 4 & 0\scriptsize{/6} & --
& 61 & 0\scriptsize{/6} & --
& 6 & 0\scriptsize{/6} & -- \\
N17\_g
& 7 & 0\scriptsize{/6} & --
& 76 & 0\scriptsize{/6} & --
& 6 & 0\scriptsize{/6} & --
& 6 & 0\scriptsize{/6} & --
& 91 & 0\scriptsize{/6} & --
& 10 & 0\scriptsize{/6} & -- \\
K13\_g
& 12 & 0\scriptsize{/6} & --
& 149 & 0\scriptsize{/6} & --
& 14 & 0\scriptsize{/6} & --
& 10 & 0\scriptsize{/6} & --
& 165 & 0\scriptsize{/6} & --
& 19 & 0\scriptsize{/6} & -- \\
F6\_p
& 21 & 0\scriptsize{/6} & --
& 667 & 0\scriptsize{/6} & --
& 44 & 0\scriptsize{/6} & --
& 90 & 0\scriptsize{/6} & --
& 616 & 0\scriptsize{/6} & --
& 51 & 0\scriptsize{/6} & -- \\
O12\_g
& 24 & 0\scriptsize{/6} & --
& 796 & 0\scriptsize{/6} & --
& 80 & 0\scriptsize{/6} & --
& 28 & 0\scriptsize{/6} & --
& 646 & 0\scriptsize{/6} & --
& 58 & 0\scriptsize{/6} & -- \\
S9\_p
& 43 & 0\scriptsize{/6} & --
& 933 & 0\scriptsize{/6} & --
& 125 & 0\scriptsize{/6} & --
& 104 & 0\scriptsize{/6} & --
& 831 & 0\scriptsize{/6} & --
& 108 & 0\scriptsize{/6} & -- \\
N16\_g
& 51 & 0\scriptsize{/6} & --
& 1,044 & 0\scriptsize{/6} & --
& 152 & 0\scriptsize{/6} & --
& 118 & 0\scriptsize{/6} & --
& 1,141 & 0\scriptsize{/6} & --
& 130 & 0\scriptsize{/6} & -- \\
S9\_g
& 124 & 0\scriptsize{/6} & --
& 1,888 & 0\scriptsize{/6} & --
& 361 & 0\scriptsize{/6} & --
& 238 & 0\scriptsize{/6} & --
& 1,982 & 0\scriptsize{/6} & --
& 334 & 0\scriptsize{/6} & -- \\
S11\_g
& 2,483 & 0\scriptsize{/6} & --
& 3,397 & 0\scriptsize{/6} & --
& 1,072 & 0\scriptsize{/6} & --
& 389 & 0\scriptsize{/6} & --
& 2,130 & 0\scriptsize{/6} & --
& 650 & 0\scriptsize{/6} & -- \\
N11\_g
& 6,277 & 0\scriptsize{/6} & --
& 8,337 & 0\scriptsize{/6} & --
& 2,082 & 0\scriptsize{/6} & --
& 6,000 & 0\scriptsize{/6} & --
& 7,370 & 0\scriptsize{/6} & --
& 2,309 & 0\scriptsize{/6} & -- \\
N9\_g
& 6,959 & 0\scriptsize{/6} & --
& 8,456 & 0\scriptsize{/6} & --
& 2,451 & 0\scriptsize{/6} & --
& 4,380 & 0\scriptsize{/6} & --
& 7,998 & 0\scriptsize{/6} & --
& 2,170 & 0\scriptsize{/6} & -- \\
N15\_g
& 6,869 & 0\scriptsize{/6} & --
& 9,548 & 0\scriptsize{/6} & --
& 4,044 & 0\scriptsize{/6} & --
& 2,531 & 0\scriptsize{/6} & --
& 5,308 & 0\scriptsize{/6} & --
& 1,986 & 0\scriptsize{/6} & -- \\
K14\_g
& - & \textbf{6}\scriptsize{/6} & 38.8
& - & \textbf{6}\scriptsize{/6} & 77.0
& 9,568 & 0\scriptsize{/6} & --
& 9,633 & 0\scriptsize{/6} & --
& 15,908 & 0\scriptsize{/6} & --
& 6,292 & 0\scriptsize{/6} & -- \\
S14\_g
& 34,873 & 0\scriptsize{/6} & --
& 39,467 & \textbf{2}\scriptsize{/6} & 3.3
& 11,680 & 0\scriptsize{/6} & --
& 15,269 & \textbf{1}\scriptsize{/6} & 47.3
& 22,695 & \textbf{1}\scriptsize{/6} & 54.8
& 6,566 & 0\scriptsize{/6} & -- \\
\midrule
Summary
& 4,442 & \textbf{6}\scriptsize{/84} & 38.8
& 5,756 & \textbf{8}\scriptsize{/84} & 58.6
& 2,263 & 0\scriptsize{/84} & --
& 2,771 & \textbf{1}\scriptsize{/84} & 47.3
& 4,781 & \textbf{1}\scriptsize{/84} & 54.8
& 1,478 & 0\scriptsize{/84} & -- \\
\bottomrule
\end{tabular}
}
\end{minipage}
\end{center}

The comparison between E$p$M and EC$p$M-SPC shows that the SPC constraints can also improve the E$p$M formulation by acting as supervalid inequalities. Because E$p$M does not enforce contiguity, it has a simpler formulation and is faster for most of the smaller road networks. For the larger networks, however, adding the SPC constraints to E$p$M results in a lower average solution time for each of the five networks for which an average E$p$M solution time is available. Across these five networks, the average solution time decreases from $11{,}492$ seconds for E$p$M to $4{,}266$ seconds for EC$p$M-SPC, an improvement factor of $2.7$. The contrast is even more pronounced for K14\_g, where E$p$M reaches the time limit in every case, but EC$p$M-SPC solves all six instances. Thus, for the larger instances, the computational benefit of the SPC constraints outweighs the additional computational burden of the larger formulation.

The proposed cuts (i.e., inequalities~\eqref{eq:VI1} and \eqref{eq:VI3}) provide their greatest benefit on the more difficult instances. The number of time-limited instances decreases from six to one for E$p$M and from eight to one for EC$p$M-CSC, while EC$p$M-SPC solves every instance both with and without the cuts. In particular, the cuts eliminate all time-limited K14\_g instances for both E$p$M and EC$p$M-CSC. They also reduce the average solution time over the solved instances by factors of $1.6$, $1.2$, and $1.5$ for E$p$M, EC$p$M-CSC, and EC$p$M-SPC, respectively.

Overall, EC$p$M-SPC with the aforementioned combination of cuts provides the best computational performance, with an average solution time of $1{,}478$ seconds and no time-limited instances. Its average solution time is lower than those of E$p$M with cuts and EC$p$M-CSC with cuts by factors of $1.9$ and $3.2$, respectively. These results show that the SPC constraints provide the strongest computational foundation among the tested formulations, while the proposed inequalities yield further performance improvements.

We also considered the B\&B\&Cut approach discussed in Section~\ref{background} for EC$p$M-CSC but did not include it in the computational study. At each iteration, B\&B\&Cut solves the current relaxation of EC$p$M, which initially reduces to E$p$M, checks the incumbent solution for contiguity violations, adds the corresponding cuts, and restarts branch-and-bound. This restart is the main difference from the B\&C approach described in Section~\ref{background}. Since the first iteration of B\&B\&Cut already requires solving E$p$M, which is slower than EC$p$M-SPC on all six largest networks, subsequent cut-generation rounds require additional solution effort and a restart of branch-and-bound each time. B\&B\&Cut is therefore not expected to be competitive on the more challenging instances.

\section{Conclusion}
\label{Conclusion}
This paper develops two binary programming formulations for the edge-based contiguous $p$-median (EC$p$M) problem. The first uses an exponential family of cut-set contiguity (CSC) constraints that are separated through a branch-and-cut algorithm. The second is a compact formulation that uses the proposed shortest-path contiguity (SPC) constraints to impose a stronger form of contiguity. This formulation can be solved directly with an off-the-shelf MILP solver. We show that E$p$M, EC$p$M, and EC$p$M-SPC have the same optimal objective value; therefore, the SPC constraints can strengthen the E$p$M formulation and improve its computational performance. We also derive three families of model-strengthening inequalities from structural properties of EC$p$M solutions. These include center-incidence supervalid inequalities that preserve at least one optimal solution, center-incidence non-valid logic cuts that preserve all optimal EC$p$M solutions, and path-based valid inequalities. We establish the conditions under which these inequalities preserve optimal or feasible solutions and develop a practical selection strategy for adding them to the formulations.

Computational experiments on $84$ instances derived from $14$ real-world road networks show that EC$p$M-SPC provides the strongest overall performance. It solves every tested instance, whereas E$p$M and EC$p$M-CSC reach the time limit on six and eight instances, respectively. Among the solved instances, EC$p$M-SPC also has substantially lower average solution times than the other two formulations, particularly on the larger road networks. The proposed model-strengthening inequalities provide further improvements, reducing the number of time-limited instances from six to one for E$p$M and from eight to one for EC$p$M-CSC. They also reduce the average solution time over the solved instances by factors of $1.6$, $1.2$, and $1.5$ for E$p$M, EC$p$M-CSC, and EC$p$M-SPC, respectively. EC$p$M-SPC with the proposed cuts achieves the best overall performance, with an average solution time of $1{,}478$ seconds and no time-limited instances. The proposed formulations also provide a foundation for incorporating additional territory requirements under which contiguity may become restrictive. Future work may investigate such extensions, as well as additional structural properties and strengthening inequalities to further improve the scalability of EC\(p\)M and related edge-based location-allocation problems.

\section*{Author contributions}
\textbf{Zeyad Kassem:} Conceptualization, Methodology, Software, Validation, Formal analysis, Investigation, Data curation, Visualization, Writing -- original draft, Writing -- review \& editing. \textbf{Behnam Jabbari-Marand:} Conceptualization, Methodology, Software, Validation, Formal analysis, Investigation, Data curation, Visualization, Writing -- original draft, Writing -- review \& editing. \textbf{Adolfo R. Escobedo:} Conceptualization, Supervision, Project administration, Writing -- review \& editing.

\section*{Declaration of competing interest}
The authors declare that they have no known competing financial interests or personal relationships that could have appeared to influence the work reported in this paper.

\appendix

\numberwithin{equation}{section}
\numberwithin{figure}{section}
\numberwithin{table}{section}

\section{Proof of Theorem \ref{thm:tight}}
\label{append:2}

\renewcommand{\thetheorem}{\getrefnumber{thm:tight}}
\begin{theorem} Define the polytope $\mathcal{P}_\text{SPC-1}=\lbrace \mathbf{(x,w)} \in \left[ 0,1 \right]^{|V|\times |E|} \times \left[ 0,1 \right]^{|V|} :\eqref{const:110}-\eqref{const:1121}, \eqref{eqn:ineq_new}\rbrace$ and the polytope $\mathcal{P}_\text{SPC-2}=\lbrace \mathbf{(x,w)} \in \left[ 0,1 \right]^{|V|\times |E|} \times \left[ 0,1 \right]^{|V|} :\eqref{const:110}-\eqref{const:1121}, \eqref{eqn:ineq_new_1}\rbrace$. For any instance of EC$p$M, $\mathcal{P}_\text{SPC-2} \subseteq \mathcal{P}_\text{SPC-1}$, and the inclusion can be strict.
\end{theorem}

\begin{proof}
Constraint \eqref{eqn:ineq_new} can be obtained as a linear combination of constraints \eqref{eqn:ineq_new_1}; specifically, for any $i \in V$ and $(j,k) \in E$, summing \eqref{eqn:ineq_new_1} over all $(l,m) \in SP_{i,(j,k)}$ yields constraint \eqref{eqn:ineq_new}. This implies that all feasible solutions to $\mathcal{P}_\text{SPC-2}$ are also feasible for $\mathcal{P}_\text{SPC-1}$ and, thus, $\mathcal{P}_\text{SPC-2} \subseteq \mathcal{P}_\text{SPC-1}$.

We show that the inclusion can be strict. Let $p\geq 2$ and assume that there are at least two nodes with a shortest path of three hops (i.e., there are three consecutive edges separating them). We construct a solution $(\mathbf{x}^\prime,\mathbf{w}^\prime)\in \mathcal{P}_\text{SPC-1} \setminus \mathcal{P}_\text{SPC-2}$. An example is depicted in Figure~\ref{Fig:proof} for convenience. Therein, squiggly lines indicate a collapsed non-empty set of edges along some indicated shortest path. A number adjacent to a dotted line indicates the value of the allocation variable associated with that edge and center node $i_1$, and a number adjacent to a solid line indicates the value of the allocation variable associated with that edge and center node $i_3$.

\begin{figure}[h]
  \begin{center}
  \begin{tikzpicture}
    \Vertex[x=-4,y=0,color=white,label=$i_0$]{A}        \Vertex[x=-2,y=0,color=white,style=dashed,label=$i_1$]{B}
    \Vertex[x=0,y=0,color=white,label=$i_2$]{C}
    \Vertex[x=2,y=0,color=black,label=\textcolor{white}{$i_3$}]{D}
    \Vertex[x=0,y=-2,color=white,label=$\hat{i_2}$]{E}		\Vertex[x=2,y=-2,color=white,label=$\hat{j_2}$]{F}
    \Vertex[x=-4,y=-2,color=white,label=$\hat{j_1}$]{H}
    \Vertex[x=-6,y=-2,color=white,label=$\hat{i_1}$]{G}
    \Vertex[x=4,y=0,color=white,label=$\hat{i_3}$]{I}
    \Vertex[x=6,y=0,color=white,label=$\hat{j_3}$]{J}
    \Edge[lw=2pt,label=0.2,style=dashed,bend=10,position=above](A)(B)
    \Edge[lw=2pt,label=0.8,bend=-10,position=below](A)(B)       \Edge[lw=2pt,label=0.3,bend=10,style=dashed,position=above](B)(C)        \Edge[lw=2pt,label=0.7,bend=-10,position=below](B)(C)
    \Edge[lw=2pt](C)(D)
    \draw[-,label=0.8,decorate,decoration=snake] (C) to[bend right=10] (E);
    \draw[-,style=dashed,decorate,decoration=snake] (C) to[bend right=-10] (E);
    \draw[-,style=dashed,label=0.8,decorate,decoration=snake] (A) to[bend right=10] (H);
    \draw[-,decorate,decoration=snake] (A) to[bend right=-10] (H);
     \Edge[lw=2pt,label=0.3,bend=10,style=dashed,position=above](E)(F)
    \Edge[lw=2pt,label=0.7,bend=-10,position=below](E)(F)
    \Edge[lw=2pt,label=0.2,bend=10,style=dashed,position=above](G)(H)
    \Edge[lw=2pt,label=0.8,bend=-10,position=below](G)(H)
    \Edge[lw=2pt](I)(J)
    \draw[-,decorate,decoration=snake] (I) to (D);
    \coordinate (dummy1) at (5,-1.5);
    \coordinate (dummy2) at (5.5,-1.5);
    \draw[black,line width=2pt] (dummy1) -- (dummy2) node[pos=1, right, font=\small, black,line width=2pt] {First Territory};
    \coordinate (dummy3) at (5,-2);
    \coordinate (dummy4) at (5.5,-2);
    \draw[dotted,black,line width=2pt] (dummy3) -- (dummy4) node[pos=1, right, font=\small, black,line width=2pt] {Second Territory};
  \end{tikzpicture}\\
  \end{center}
  \caption{Subgraph depicting fractional solutions for the linear relaxation of SPC-1 and SPC-2.} \label{Fig:proof}
\end{figure}

Without loss of generality, consider the shortest path from $i_0$ to $i_3$, with $|SP_{i_0,i_3}|=3$, and fix $i_1$ and $i_3$ as center nodes ($w_{i_3}=w_{i_1}=1$). If $p>2$, select the remaining $(p-2)$ centers arbitrarily from $V \setminus \lbrace i_1,i_3 \rbrace$; then set the allocation variables corresponding to edges $(i_0,i_1)$, $(i_1,i_2)$, and $(i_2,i_3)$ as follows:\bla

\begin{itemize}
  \item \underline{Edge $(i_0,i_1)$}. Set $x_{i_1,(i_0,i_1)}=0.2$ and $x_{i_3,(i_0,i_1)}=0.8$.
  \item \underline{Edge $(i_1,i_2)$}. Set $x_{i_1,(i_1,i_2)}=0.3$ and $x_{i_3,(i_1,i_2)}=0.7$.
  \item \underline{Edge $(i_2,i_3)$}. Set $x_{i_3,(i_2,i_3)}=1$.
\end{itemize}

For the three edges above, consider the following three mutually exclusive cases.

\noindent\textbf{Case 1:}
\underline{$(i_0,i_1) \in SP_{i_1,(\hat{i_1},\hat{j_1})}$}. Set $x_{i_1,(\hat{i_1},\hat{j_1})}=0.2$ and $x_{i_1,(l,m)}=0.2 \ \forall (l,m) \in SP_{i_0,(\hat{i_1},\hat{j_1})}$; additionally, set $x_{i_3,(\hat{i_1},\hat{j_1})}=0.8$, and $x_{i_3,(l,m)}=0.8 \ \forall (l,m) \in SP_{i_0,(\hat{i_1},\hat{j_1})}$.

\noindent\textbf{Case 2:}
\underline{$(i_1,i_2) \in SP_{i_1,(\hat{i_2},\hat{j_2})}$ and $(i_2,i_3) \notin SP_{i_1,(\hat{i_2},\hat{j_2})}$}. Set $x_{i_1,(\hat{i_2},\hat{j_2})}=0.3$ and $x_{i_1,(l,m)}=0.3 \ \forall (l,m) \in SP_{i_2,(\hat{i_2},\hat{j_2})}$; additionally, set $x_{i_3,(\hat{i_2},\hat{j_2})}=0.7$ and $x_{i_3,(l,m)}=0.7 \ \forall (l,m) \in SP_{i_2,(\hat{i_2},\hat{j_2})}$.

\noindent\textbf{Case 3:} \underline{$(i_1,i_2), (i_2,i_3) \in SP_{i_1,(\hat{i_3},\hat{j_3})}$}. Set $x_{i_3,(\hat{i_3},\hat{j_3})}=1$ and $x_{i_3,(l,m)}=1 \ \forall (l,m) \in SP_{i_3,(\hat{i_3},\hat{j_3})}$. To finish constructing the solution, allocate all other edges in full to center $i_1$.

The constructed solution $(\mathbf{x}^\prime,\mathbf{w}^\prime)$ satisfies constraints~\eqref{const:110}--\eqref{const:1121} and \eqref{eqn:ineq_new}: The sum of the allocation variables for each edge is equal to 1; the number of territories is equal to $p$; and each edge is allocated to a node, only if the node is chosen as a center node.

Now, isolating constraint \eqref{eqn:ineq_new} for edge $(i_0,i_1)$ and node $i_3$, the chosen variable settings give that
\begin{align}
 |SP_{i_3,(i_0,i_1)}| x_{i_3,\left(i_0,i_1\right)}= 2 \cdot (0.8) \leq 1+0.7 = x_{i_3,\left(i_1,i_2\right)} + x_{i_3,\left(i_2,i_3\right)}
\end{align}
and, thus, the constraint is satisfied. On the other hand, constraint \eqref{eqn:ineq_new_1} is violated since
\begin{align}
 x_{i_3, (i_0,i_1)}=0.8 \nleq 0.7 = x_{i_3, (i_1,i_2)}.
\end{align}
\end{proof}\vspace{-16pt}
\section{Proof of Theorem \ref{thm:tighter}}
\label{append:3}
\renewcommand{\thetheorem}{\getrefnumber{thm:tighter}}
\begin{theorem}
Let $\mathcal{P}_\text{SPC-2}$ be the polytope defined in Theorem~\ref{thm:tight} and define the polytope $\mathcal{P}_\text{SPC-3}=\{\mathbf{(x,w)} \in [0,1]^{|V|\times |E|} \times [0,1]^{|V|} : \eqref{const:110}-\eqref{const:1121}, \eqref{eqn:ineq_new_2}\}$. Then, $\mathcal{P}_\text{SPC-2} = \mathcal{P}_\text{SPC-3}$.
\end{theorem}

\begin{proof}
Constraints~\eqref{eqn:ineq_new_2} enforce that, if edge $(j,k)$ is assigned to center node $i$, then the preceding edge $(k,l)$ in $SP_{i,(j,k)}$ must also be assigned to $i$. By the same logic, if $(k,l)$ is assigned to $i$, then the preceding edge $(l,o)$ on the shortest path $SP_{i,(k,l)}$ must also be assigned to $i$. By transitivity, if $x_{i,(j,k)} \leq x_{i,(k,l)}$ and $x_{i,(k,l)} \leq x_{i,(l,o)}$, then $x_{i,(j,k)} \leq x_{i,(l,o)}$. The same logic can be extended to all adjacent pairwise edges along $SP_{i,(j,k)}$. Constraints~\eqref{eqn:ineq_new_1} enforce exactly the same logic, but they list all inequalities explicitly (i.e., without taking advantage of the implied transitivity). This means that $\mathcal{P}_\text{SPC-2}$ and $\mathcal{P}_\text{SPC-3}$ are equal.
\end{proof}

\section*{Data Availability}
The data and code supporting the findings of this study are openly available at \url{https://github.com/zekassem/ECPM-EBD-EPM/}.

\bibliographystyle{plainnat}
\bibliography{AlkhateebRef}

@Article{SalazarAguilar2011,
  author    = {Salazar-Aguilar, Mar{\'\i}a Ang{\'e}lica and R{\'\i}os-Mercado, Roger Z and Cabrera-R{\'\i}os, Mauricio},
  journal   = {Networks and Spatial Economics},
  title     = {New models for commercial territory design},
  year      = {2011},
  number    = {3},
  pages     = {487--507},
  volume    = {11},
  publisher = {Springer},
}

@Article{Butsch2014,
  author    = {Butsch, Alexander and Kalcsics, J{\"o}rg and Laporte, Gilbert},
  journal   = {INFORMS Journal on Computing},
  title     = {Districting for arc routing},
  year      = {2014},
  number    = {4},
  pages     = {809--824},
  volume    = {26},
  publisher = {INFORMS},
}

@Article{Shirabe2009,
  author    = {Shirabe, Takeshi},
  journal   = {Environment and Planning B: Planning and Design},
  title     = {Districting modeling with exact contiguity constraints},
  year      = {2009},
  number    = {6},
  pages     = {1053--1066},
  volume    = {36},
  publisher = {SAGE Publications Sage UK: London, England},
}

@InProceedings{Hooker1994a,
author="Hooker, J. N.",
editor="Borning, Alan",
title="Logic-based methods for optimization",
booktitle="Principles and Practice of Constraint Programming",
year="1994",
publisher="Springer, Berlin, Heidelberg",
address={},
pages="336--349",
series="Lecture Notes in Computer Science",
volume="874"
}

@ARTICLE{Osorio1999,
    author = {M. A. {Osorio
Lama} and R. {M\'{u}jica Garc\'\i a}},
    title = {Logic Cuts Generation in a Branch and Cut Framework for Location Problems},
    journal = {Investigaci{\'o}n Operativa},
    year = {1999},
    volume = {8},
    number={1-3},
    pages = {155--166}
}

@book{Ahuja1988,
  title     = {Network Flows: Theory, Algorithms, and Applications},
  author    = {Ahuja, Ravindra K. and Magnanti, Thomas L. and Orlin, James B.},
  year      = {1993},
  publisher = {Prentice Hall},
  address   = {Englewood Cliffs, NJ}
}

@incollection{kalcsics2020districting,
  title={Districting problems},
  author={Kalcsics, J{\"o}rg and R{\'\i}os-Mercado, Roger Z},
  booktitle={Location science},
  pages={705--743},
  year={2020},
  publisher={Springer}
}

@article{lespay2022territory,
  title={Territory design for the multi-period vehicle routing problem with time windows},
  author={Lespay, Hern{\'a}n and Suchan, Karol},
  journal={Computers \& Operations Research},
  volume={145},
  pages={105866},
  year={2022},
  publisher={Elsevier}
}

@Article{GarciaAyala2016,
  author    = {Garc{\'\i}a-Ayala, Gabriela and Gonz{\'a}lez-Velarde, Jos{\'e} Luis and R{\'\i}os-Mercado, Roger Z and Fern{\'a}ndez, Elena},
  journal   = {International Transactions in Operational Research},
  title     = {A novel model for arc territory design: Promoting Eulerian districts},
  year      = {2016},
  number    = {3},
  pages     = {433--458},
  volume    = {23},
  publisher = {Wiley Online Library},
}

@Article{Osorio2003,
  author    = {Osorio, Mar{\i}?a A and Laguna, Manuel},
  journal   = {European Journal of Operational Research},
  title     = {Logic cuts for multilevel generalized assignment problems},
  year      = {2003},
  number    = {1},
  pages     = {238--246},
  volume    = {151},
  publisher = {Elsevier},
}

@Article{Brimberg2008,
  author    = {Brimberg, Jack and Hansen, Pierre and Mladenovic, Nenad and Salhi, Said},
  journal   = {International Journal of Operations Research},
  title     = {A survey of solution methods for the continuous location-allocation problem},
  year      = {2008},
  number    = {1},
  pages     = {1--12},
  volume    = {5},
  publisher = {Operations Research Society of Taiwan, ORSTW},
}

@Article{Tansel1983,
  author    = {Tansel, Barbaros C and Francis, Richard L and Lowe, Timothy J},
  journal   = {Management science},
  title     = {State of the art-location on networks: a survey. Part I: the p-center and p-median problems},
  year      = {1983},
  number    = {4},
  pages     = {482--497},
  volume    = {29},
  publisher = {INFORMS},
}

@Article{Hakimi1964,
  author    = {Hakimi, S Louis},
  journal   = {Operations research},
  title     = {Optimum locations of switching centers and the absolute centers and medians of a graph},
  year      = {1964},
  number    = {3},
  pages     = {450--459},
  volume    = {12},
  publisher = {informs},
}

@Article{Wang2021,
  author    = {Wang, Dian and Zhao, Jun and D'Ariano, Andrea and Peng, Qiyuan},
  journal   = {European Journal of Operational Research},
  title     = {Simultaneous node and link districting in transportation networks: Model, algorithms and railway application},
  year      = {2021},
  number    = {1},
  pages     = {73--94},
  volume    = {292},
  publisher = {Elsevier},
}

@article{mendes2022capacitated,
  title={The capacitated and economic districting problem},
  author={Mendes, Luis Henrique Pauleti and Usberti, Fabio Luiz and Cavellucci, Celso},
  journal={INFORMS Journal on Computing},
  volume={34},
  number={4},
  pages={2003--2016},
  year={2022},
  publisher={INFORMS}
}

@article{bosona2020urban,
  title={Urban freight last mile logistics—challenges and opportunities to improve sustainability: a literature review},
  author={Bosona, Techane},
  journal={Sustainability},
  volume={12},
  number={21},
  pages={8769},
  year={2020},
  publisher={MDPI}
}

@article{wygonik2018urban,
  title={Urban form and last-mile goods movement: Factors affecting vehicle miles travelled and emissions},
  author={Wygonik, Erica and Goodchild, Anne V},
  journal={Transportation Research Part D: Transport and Environment},
  volume={61},
  pages={217--229},
  year={2018},
  publisher={Elsevier}
}

@article{alumur2021perspectives,
  title={Perspectives on modeling hub location problems},
  author={Alumur, Sibel A and Campbell, James F and Contreras, Ivan and Kara, Bahar Y and Marianov, Vladimir and O’Kelly, Morton E},
  journal={European Journal of Operational Research},
  volume={291},
  number={1},
  pages={1--17},
  year={2021},
  publisher={Elsevier}
}

@article{farughi2020novel,
  title={A novel optimization model for designing compact, balanced, and contiguous healthcare districts},
  author={Farughi, Hiwa and Tavana, Madjid and Mostafayi, Sobhan and Santos Arteaga, Francisco J},
  journal={Journal of the Operational Research Society},
  volume={71},
  number={11},
  pages={1740--1759},
  year={2020},
  publisher={Taylor \& Francis}
}

@article{barbato2024hamiltonian,
  title={The Hamiltonian p-median problem: Polyhedral results and branch-and-cut algorithms},
  author={Barbato, Michele and Gouveia, Lu{\'\i}s},
  journal={European Journal of Operational Research},
  volume={316},
  number={2},
  pages={473--487},
  year={2024},
  publisher={Elsevier}
}

@article{zoltners1983sales,
  title={Sales territory alignment: A review and model},
  author={Zoltners, Andris A and Sinha, Prabhakant},
  journal={Management Science},
  volume={29},
  number={11},
  pages={1237--1256},
  year={1983},
  publisher={INFORMS}
}

@book{merris2011graph,
  title={Graph theory},
  author={Merris, Russell},
  address= {New York, NY, USA},
  year={2011},
  publisher={John Wiley \& Sons}
}

@book{diestel2025,
  author    = {Reinhard Diestel},
  title     = {Graph Theory},
  edition   = {6th},
  year      = {2025},
  publisher = {Springer}
}

@article{mehrotra1998optimization,
  title={An optimization based heuristic for political districting},
  author={Mehrotra, Anuj and Johnson, Ellis L and Nemhauser, George L},
  journal={Management Science},
  volume={44},
  number={8},
  pages={1100--1114},
  year={1998},
  publisher={INFORMS}
}

@article{kassem2023models,
  title={Models and network insights for edge-based districting with simultaneous location-allocation decisions},
  author={Kassem, Zeyad and Escobedo, Adolfo R},
  journal={IISE Transactions},
  volume={55},
  number={8},
  pages={768--780},
  year={2023},
  publisher={Taylor \& Francis}
}

@article{cova2000contiguity,
  title={Contiguity constraints for single-region site search problems},
  author={Cova, Thomas J and Church, Richard L},
  journal={Geographical Analysis},
  volume={32},
  number={4},
  pages={306--329},
  year={2000},
  publisher={Wiley Online Library}
}

@article{kiilerich2018new,
  title={New large-scale data instances for CARP and new variations of CARP},
  author={Kiilerich, Lone and W{\o}hlk, Sanne},
  journal={INFOR: Information Systems and Operational Research},
  volume={56},
  number={1},
  pages={1--32},
  year={2018},
  publisher={Taylor \& Francis}
}

@article{dantzig1954solution,
  title={Solution of a large-scale traveling-salesman problem},
  author={Dantzig, George and Fulkerson, Ray and Johnson, Selmer},
  journal={Journal of the operations research society of America},
  volume={2},
  number={4},
  pages={393--410},
  year={1954},
  publisher={INFORMS}
}

@article{williams2002zero,
  title={A zero-one programming model for contiguous land acquisition},
  author={Williams, Justin C},
  journal={Geographical Analysis},
  volume={34},
  number={4},
  pages={330--349},
  year={2002},
  publisher={Wiley Online Library}
}

@article{garfinkel1970optimal,
  title={Optimal political districting by implicit enumeration techniques},
  author={Garfinkel, Robert S and Nemhauser, George L},
  journal={Management Science},
  volume={16},
  number={8},
  pages={B-495–B-508},
  year={1970},
  publisher={INFORMS}
}

@article{carvajal2013imposing,
  title={Imposing connectivity constraints in forest planning models},
  author={Carvajal, Rodolfo and Constantino, Miguel and Goycoolea, Marcos and Vielma, Juan Pablo and Weintraub, Andr{\'e}s},
  journal={Operations Research},
  volume={61},
  number={4},
  pages={824--836},
  year={2013},
  publisher={INFORMS}
}

@article{canas2025location,
  title={On location-routing formulations for the Hamiltonian {$p$}-median problem},
  author={Canas, Francisco and Gouveia, Lu{\'\i}s},
  journal={European Journal of Operational Research},
  volume={331},
  number={1},
  pages={62--76},
  year={2026},
  doi={10.1016/j.ejor.2025.10.023},
  publisher={Elsevier}
}

@article{validi2022imposing,
  title={Imposing contiguity constraints in political districting models},
  author={Validi, Hamidreza and Buchanan, Austin and Lykhovyd, Eugene},
  journal={Operations Research},
  volume={70},
  number={2},
  pages={867--892},
  year={2022},
  publisher={INFORMS}
}

@article{jafari2013new,
  title={A new method to solve the fully connected reserve network design problem},
  author={Jafari, Nahid and Hearne, John},
  journal={European Journal of Operational Research},
  volume={231},
  number={1},
  pages={202--209},
  year={2013},
  publisher={Elsevier}
}

@article{onal2016optimal,
  title={Optimal design of compact and functionally contiguous conservation management areas},
  author={{\"O}nal, Hayri and Wang, Yicheng and Dissanayake, Sahan TM and Westervelt, James D},
  journal={European Journal of Operational Research},
  volume={251},
  number={3},
  pages={957--968},
  year={2016},
  publisher={Elsevier}
}

@article{ricca2008local,
  title={Local search algorithms for political districting},
  author={Ricca, Federica and Simeone, Bruno},
  journal={European Journal of Operational Research},
  volume={189},
  number={3},
  pages={1409--1426},
  year={2008},
  publisher={Elsevier}
}

@article{bergmann2020integrating,
  title={Integrating first-mile pickup and last-mile delivery on shared vehicle routes for efficient urban e-commerce distribution},
  author={Bergmann, Felix M and Wagner, Stephan M and Winkenbach, Matthias},
  journal={Transportation Research Part B: Methodological},
  volume={131},
  pages={26--62},
  year={2020},
  publisher={Elsevier}
}

@article{dablanc2010impacts,
  title={The impacts of logistics sprawl: How does the location of parcel transport terminals affect the energy efficiency of goods’ movements in Paris and what can we do about it?},
  author={Dablanc, Laetitia and Rakotonarivo, Dina},
  journal={Procedia-Social and Behavioral Sciences},
  volume={2},
  number={3},
  pages={6087--6096},
  year={2010},
  publisher={Elsevier}
}

@article{JabbariMarand2026,
  author  = {Jabbari-Marand, Behnam and Escobedo, Adolfo R.},
  title   = {Polynomial-time algorithms for setting tight {big-M} coefficients in transmission expansion planning with disconnected buses},
  journal = {Optimization Letters},
  year    = {2026},
  month   = sep,
  doi     = {10.1007/s11590-026-02333-6},
  issn    = {1862-4480}
}

@article{jost2025districting,
  title={Why Districting Becomes NP-hard},
  author={Jost, Niklas and Escobedo, Adolfo and Kirchheim, Alice},
  journal={arXiv preprint arXiv:2510.25614},
  year={2025}
}

@InCollection{Marianov2002,
  author    = {Marianov, Vladimir and Serra, Daniel},
  title     = {Location Problems in the Public Sector},
  booktitle = {Facility Location: Applications and Theory},
  editor    = {Drezner, Zvi and Hamacher, Horst W.},
  chapter   = {4},
  pages     = {119--150},
  publisher = {Springer},
  address   = {Berlin, Heidelberg},
  year      = {2002},
  doi       = {10.1007/978-3-642-56082-8_4}
}

@inproceedings{zhang2017graph,
  title={Graph edge partitioning via neighborhood heuristic},
  author={Zhang, Chenzi and Wei, Fan and Liu, Qin and Tang, Zhihao Gavin and Li, Zhenguo},
  booktitle={Proceedings of the 23rd ACM SIGKDD International Conference on Knowledge Discovery and Data Mining},
  pages={605--614},
  year={2017}
}

@incollection{Daskin2013,
  author    = {Daskin, Mark S.},
  title     = {Median problems},
  booktitle = {Network and Discrete Location: Models, Algorithms, and Applications},
  edition   = {second},
  chapter   = {6},
  pages     = {235--293},
  publisher = {John Wiley \& Sons, Ltd},
  year      = {2013},
  isbn      = {9781118537015},
  doi       = {10.1002/9781118537015.ch06}
}

@book{drezner2004facility,
  title={Facility location: applications and theory},
  author={Drezner, Zvi and Hamacher, Horst W},
  year={2004},
  publisher={Springer Science \& Business Media}
}

@article{owen1998strategic,
  title={Strategic facility location: A review},
  author={Owen, Susan Hesse and Daskin, Mark S},
  journal={European journal of operational research},
  volume={111},
  number={3},
  pages={423--447},
  year={1998},
  publisher={Elsevier}
}

@article{melo2009facility,
  title={Facility location and supply chain management--A review},
  author={Melo, M Teresa and Nickel, Stefan and Saldanha-Da-Gama, Francisco},
  journal={European journal of operational research},
  volume={196},
  number={2},
  pages={401--412},
  year={2009},
  publisher={Elsevier}
}

@article{croci2023balanced,
  title={The balanced p-median problem with unitary demand},
  author={Croci, Davide and Jabali, Ola and Malucelli, Federico},
  journal={Computers \& Operations Research},
  volume={155},
  pages={106242},
  year={2023},
  publisher={Elsevier}
}

@inproceedings{yen2010connected,
  title={The connected p-median problem on graphs with forbidden vertices},
  author={Yen, William Chung-Kung and Chang, Shun-Chieh and Yang, Shung Chih-Shiang and Hu, Shou-Cheng},
  booktitle={Proceedings of 27th Workshop on Combinatorial Mathematics and Computation Theory, Providence University, Taichung, Taiwan},
  pages={168--173},
  year={2010},
  organization={Citeseer}
}

@article{branco1990hamiltonian,
  title={The Hamiltonian p-median problem},
  author={Branco, IM and Coelho, JD},
  journal={European Journal of Operational Research},
  volume={47},
  number={1},
  pages={86--95},
  year={1990},
  publisher={Elsevier}
}

@article{canovas2007strengthened,
  title={A strengthened formulation for the simple plant location problem with order},
  author={C{\'a}novas, L{\'a}zaro and Garc{\'\i}a, Sergio and Labb{\'e}, Martine and Mar{\'\i}n, Alfredo},
  journal={Operations Research Letters},
  volume={35},
  number={2},
  pages={141--150},
  year={2007},
  publisher={Elsevier}
}

@article{ulukan2015survey,
  title={A survey of discrete facility location problems},
  author={Ulukan, Z and Demircio{\u{g}}lu, E},
  journal={International Journal of Industrial and Manufacturing Engineering},
  volume={9},
  number={7},
  pages={2487--2492},
  year={2015}
}

@incollection{martin2001general,
  title={General mixed integer programming: Computational issues for branch-and-cut algorithms},
  author={Martin, Alexander},
  booktitle={Computational combinatorial optimization: Optimal or provably near-optimal solutions},
  pages={1--25},
  year={2001},
  publisher={Springer}
}

@inproceedings{gollowitzer2011new,
  title={New models for and numerical tests of the Hamiltonian p-median problem},
  author={Gollowitzer, Stefan and Pereira, Dilson Lucas and Wojciechowski, Adam},
  booktitle={International conference on network optimization},
  pages={385--394},
  year={2011},
  organization={Springer}
}

@article{drexl1999fast,
  title={Fast approximation methods for sales force deployment},
  author={Drexl, Andreas and Haase, Knut},
  journal={Management Science},
  volume={45},
  number={10},
  pages={1307--1323},
  year={1999},
  publisher={INFORMS}
}

@article{lolonis1993location,
  title={Location-allocation models as decision aids in delineating administrative regions},
  author={Lolonis, Panagiotis and Armstrong, Marc P},
  journal={Computers, environment and urban systems},
  volume={17},
  number={2},
  pages={153--174},
  year={1993},
  publisher={Elsevier}
}

@inproceedings{mayer2018adwise,
  title={Adwise: Adaptive window-based streaming edge partitioning for high-speed graph processing},
  author={Mayer, Christian and Mayer, Ruben and Tariq, Muhammad Adnan and Geppert, Heiko and Laich, Larissa and Rieger, Lukas and Rothermel, Kurt},
  booktitle={2018 IEEE 38th International Conference on Distributed Computing Systems (ICDCS)},
  pages={685--695},
  year={2018},
  organization={IEEE}
}

@inproceedings{mayer2021hybrid,
  title={Hybrid edge partitioner: Partitioning large power-law graphs under memory constraints},
  author={Mayer, Ruben and Jacobsen, Hans-Arno},
  booktitle={Proceedings of the 2021 International Conference on Management of Data},
  pages={1289--1302},
  year={2021}
}

@article{israeli2002shortest,
  title={Shortest-path network interdiction},
  author={Israeli, Eitan and Wood, R Kevin},
  journal={Networks: An International Journal},
  volume={40},
  number={2},
  pages={97--111},
  year={2002},
  publisher={Wiley Online Library}
}

@article{alvarez2021districting,
  title={A districting application with a quality of service objective},
  author={{\'A}lvarez-Miranda, Eduardo and Pereira, Jordi},
  journal={Mathematics},
  volume={10},
  number={1},
  pages={13},
  year={2022},
  publisher={MDPI}
}

@article{alvarez2025multi,
  title={A multi-criteria districting approach with a lexicographic compactness metric: An application to the Chilean postal service},
  author={{\'A}lvarez-Miranda, Eduardo and Epstein, Rafael and Pereira, Jordi and Sinnl, Markus and Urrutia, Rodolfo},
  journal={Computers \& Operations Research},
  volume={173},
  pages={106845},
  year={2025},
  publisher={Elsevier}
}

@article{wei2024supervalid,
  title={On supervalid inequalities for binary interdiction games},
  author={Wei, Ningji and Walteros, Jose L},
  journal={Mathematical Programming},
  volume={212},
  number={1},
  pages={437--478},
  year={2025},
  publisher={Springer}
}

@article{baycik2019robust,
  title={Robust location of hidden interdictions on a shortest path network},
  author={Baycik, N Orkun and Sullivan, Kelly M},
  journal={IISE Transactions},
  volume={51},
  number={12},
  pages={1332--1347},
  year={2019},
  publisher={Taylor \& Francis}
}

@article{yuan2021mixed,
  title={Mixed integer programming formulations for the generalized traveling salesman problem with time windows},
  author={Yuan, Yuan and Cattaruzza, Diego and Ogier, Maxime and Rousselot, Cyriaque and Semet, Fr{\'e}d{\'e}ric},
  journal={4OR},
  volume={19},
  pages={571--592},
  year={2021},
  publisher={Springer}
}

@article{sefair2017defender,
  title={A defender-attacker model and algorithm for maximizing weighted expected hitting time with application to conservation planning},
  author={Sefair, Jorge A and Smith, J Cole and Acevedo, Miguel A and Fletcher Jr, Robert J},
  journal={IISE Transactions},
  volume={49},
  number={12},
  pages={1112--1128},
  year={2017},
  publisher={Taylor \& Francis}
}

@article{yuan2020note,
  title={A note on the lifted Miller--Tucker--Zemlin subtour elimination constraints for routing problems with time windows},
  author={Yuan, Yuan and Cattaruzza, Diego and Ogier, Maxime and Semet, Fr{\'e}d{\'e}ric},
  journal={Operations Research Letters},
  volume={48},
  number={2},
  pages={167--169},
  year={2020},
  publisher={Elsevier}
}

@article{ceselli2005branch,
  title={A branch-and-price algorithm for the capacitated p-median problem},
  author={Ceselli, Alberto and Righini, Giovanni},
  journal={Networks: An International Journal},
  volume={45},
  number={3},
  pages={125--142},
  year={2005},
  publisher={Wiley Online Library}
}

@article{guden2019dynamic,
  title={The dynamic p-median problem with mobile facilities},
  author={G{\"u}den, H{\"u}seyin and S{\"u}ral, Haldun},
  journal={Computers \& Industrial Engineering},
  volume={135},
  pages={615--627},
  year={2019},
  publisher={Elsevier}
}

@article{hess1965nonpartisan,
  title={Nonpartisan political redistricting by computer},
  author={Hess, Sidney Wayne and Weaver, JB and Siegfeldt, HJ and Whelan, JN and Zitlau, PA},
  journal={Operations Research},
  volume={13},
  number={6},
  pages={998--1006},
  year={1965},
  publisher={INFORMS}
}

@article{ricca2011political,
  title={Political districting: from classical models to recent approaches},
  author={Ricca, Federica and Scozzari, Andrea and Simeone, Bruno},
  journal={4OR},
  volume={9},
  pages={223--254},
  year={2011},
  publisher={Springer}
}

@article{hakimi1979algorithmic,
  title={An algorithmic approach to network location problems II: The p-medians},
  author = {Kariv, O. and Hakimi, S. L.},
  journal={SIAM Journal on Applied Mathematics},
  volume={37},
  number={3},
  pages={539--560},
  year={1979}
}

\end{document}